\documentclass[article]{jss}

\usepackage{thumbpdf,lmodern}

\usepackage[utf8]{inputenc}

\usepackage{mathtools}
\usepackage{float}
\usepackage{amssymb, amsmath, amsthm}
\floatstyle{ruled} 

\usepackage{caption}
\usepackage{subcaption}
\usepackage{algorithm}
\usepackage{algorithmicx}
\usepackage[noend]{algpseudocode}

\algnewcommand\algorithmicinput{\textbf{Input:}}
\algnewcommand\INPUT{\item[\algorithmicinput]}
\algnewcommand\algorithmicoutput{\textbf{Output:}}
\algnewcommand\OUTPUT{\item[\algorithmicoutput]}

\usepackage{tikz}
\usetikzlibrary{positioning, arrows.meta, shapes.geometric}
\tikzset{%
  dot/.style n args = {4}{name=#3, circle, draw, inner sep=1pt, minimum size=5pt, fill=black, label={[shift={(#1,#2)}]#4:$#3$}},
  lat/.style n args = {4}{name=#3, circle, draw, inner sep=1pt, minimum size=5pt, label={[shift={(#1,#2)}]#4:$#3$}},
  sb/.style n args = {4}{name=#3, circle, draw, inner sep=1pt, minimum size=7pt, label={[shift={(#1,#2)}]#4:$#3$}},
  dot5/.style n args = {5}{name=#3, circle, draw, inner sep=1pt, minimum size=5pt, fill=black, label={[shift={(#1,#2)}]#4:$#5$}},
  lat5/.style n args = {5}{name=#3, circle, draw, inner sep=1pt, minimum size=5pt, label={[shift={(#1,#2)}]#4:$#5$}},
  sq/.style n args = {4}{name=#3, rectangle, draw, inner sep=1pt, minimum size=5pt, fill=black, label={[shift={(#1,#2)}]#4:$#3$}},
  tr/.style n args = {4}{name=#3, regular polygon,regular polygon sides=4, draw, inner sep=1pt, minimum size=6pt, fill=gray, label={[shift={(#1,#2)}]#4:$#3$}},
  bordered/.style = {draw,outer sep=1, inner sep=2, minimum size=5pt},
  >={Latex[width=1.5mm,length=2mm]},
  every picture/.style={semithick}
}

\newtheorem{theorem}{Theorem}
\newtheorem{definition}{Definition}
\newtheorem{example}{Example}
\newtheorem{lemma}{Lemma}

\newcommand{\+}[1]{\ensuremath{\mathbf{#1}}}

\newcommand{\given}{{ \, | \, }}
\newcommand{\doublebar}{{ \, || \, }}

\newcommand{\NA}{\textrm{NA}}

\newcommand\independent{\protect\mathpalette{\protect\independenT}{\perp}} 
\def\independenT#1#2{\mathrel{\rlap{$#1#2$}\mkern2mu{#1#2}}}

\newcommand{\doo}{\textrm{do}}

\newcommand{\true}{{t}}

\newcommand{\proxpar}{{(* \to \true)}}

\newcommand{\parprox}{{(\true \to *)}}

\newcommand{\dosearch}{\textsf{do-search}}

\author{Santtu Tikka\\University of Jyvaskyla
   \And Antti Hyttinen\\University of Helsinki
   \And Juha Karvanen\\University of Jyvaskyla}
\Plainauthor{Santtu Tikka, Antti Hyttinen, Juha Karvanen}

\title{Causal Effect Identification from\\ Multiple Incomplete Data
Sources:\\ A General Search-based Approach}
\Plaintitle{Causal Effect Identification from Multiple Incomplete Data
Sources: A General Search-based Approach}
\Shorttitle{Causal Effect Identification from Multiple Incomplete Data
Sources}

\Abstract{
Causal effect identification considers whether an interventional probability
distribution can be uniquely determined without parametric assumptions from
measured source distributions and structural knowledge on the generating
system.  While complete graphical criteria and procedures exist for many
identification problems, there are still challenging but important
extensions that have not been considered in the literature such as combined
transportability and selection bias, or multiple sources of selection bias. 
To tackle these new settings, we present a search algorithm directly over
the rules of do-calculus.  Due to the generality of do-calculus, the search
is capable of taking more advanced data-generating mechanisms into account
along with an arbitrary type of both observational and experimental source
distributions.  The search is enhanced via a heuristic and search space
reduction techniques.  The approach, called \dosearch{}, is provably sound,
and it is complete with respect to identifiability problems that have been
shown to be completely characterized by do-calculus.  When extended with
additional rules, the search is capable of handling missing data problems as
well.  With the versatile search, we are able to approach new problems for
which no other algorithmic solutions exist.  We perform a systematic
analysis of bivariate missing data problems and study causal inference under
case-control design.  We also present the \proglang{R}~package
\pkg{dosearch} that provides an interface for a \proglang{C++}
implementation of the search.
}

\Plainkeywords{causality, do-calculus, selection bias, transportability,
missing data, case-control design, meta-analysis}

\Address{
  Santtu Tikka (corresponding author)\\
  Department of Mathematics and Statistics\\
  Faculty of Mathematics and Science\\
  University of Jyvaskyla\\
  P.O.Box 35, FI-40014, Finland\\
  E-mail: \email{santtu.tikka@jyu.fi}\\
}
\Year{2021}
\Month{August}
\Volume{99}
\Issue{5}
\DOI{10.18637/jss.v099.i05}
\Submitdate{2019-12-05}
\Acceptdate{2020-10-26}

\begin{document}

\section{Introduction}

In many fields of science, a primary interest is determining causal effects,
that is, distributions $P(\+Y \mid \doo(\+X),\+Z)$, where variables $\+ Y$
are observed, variables $\+ X$ are intervened upon (forced to values
irrespective of their natural causes) and variables $\+ Z$ are conditioned
on \citep{Pearl:book2009}.  In this paper, instead of placing various
parametric restrictions based on background knowledge, we are interested in
the question of identifiability: can the causal effect be uniquely
determined from the distributions (data) we have and a graph representing
our structural knowledge on the generating causal system.

In the most basic setting we are identifying causal effects from a single
observational input distribution, corresponding to passively observed data. 
To solve such problems more generally than what is possible with the
back-door adjustment \citep{SGS,Pearl:book2009,greenland1999},
\citet{pearl1995causal} introduced \emph{do-calculus}, a set of three rules
that together with probability theory enable the manipulation of
interventional distributions.  \citet{Shpitser} and
\citet{huangvaltorta:complete} showed that do-calculus is complete by
presenting polynomial-time algorithms whose each step can be seen as a rule
of do-calculus or as an operation based on basic probability theory.  The
algorithms have a high practical value because the rules of do-calculus do
not by themselves provide an indication on the order in which the rules
should be applied.  The algorithms save us from manual application of
do-calculus, which is a tedious task in all but the simplest problems.

Since then many extensions of the basic identifiability problem have
appeared.  In identifiability using surrogate experiments
\citep{Bareinboim:zidentifiability}, or $z$-identifiability, an experimental
distribution is available in addition to the observed probability
distribution.  For data observed in the presence of selection bias, both
algorithmic and graphical identifiability results have been derived
\citep{bareinboim2015recovering, Correa2018}.  More generally, the presence
of missing data necessitates the representation of the missingness
mechanism, which poses additional challenges \citep{Mohan2013,
Shpitser2015,bhattacharya2019}.  Another dimension of complexity is the
number of available data sources.  Identification from a mixture of
observational and interventional distributions that originate from multiple
conceptual domains is known as transportability for which complete solutions
exist in a specific setting \citep{bareinboim2014transportability}.

While completeness has been accomplished for a number of basic
identifiability problems, there are still many challenging but important
extensions to the identifiability problem that have not been studied so far. 
Table~\ref{tab:stateofart} recaps the current state of the art
identifiability results; it also describes generalizations that we aim to
investigate in this paper.  To find solutions to the more complicated
identifiability problems, we present a unified approach to the
identification of observational and interventional causal queries by
constructing a search algorithm that directly applies the rules of
do-calculus.  We impose no restrictions on the number or type of known input
distributions: we thus provide a solution to problems for which no other
algorithmic solutions exist (Row~8 in Table~\ref{tab:stateofart}).  We also
extend to identifiability under missing data together with mechanisms
related to selection bias and transportability (Row~11 in
Table~\ref{tab:stateofart}).

The following introductory example does not fall under any of the previously
solved special cases of Table~\ref{tab:stateofart}, thus necessitating our
approach.  Consider human resource management analyzing the remuneration
policy in a company.  The graph of Figure~\ref{fig:hr} shows the key
variables.  The salary ($Y$) of an employee consist of a base salary (not
modeled explicitly) and a bonus ($B$).  The base salary depends on education
($E$) and performance ($X$) of the employee.  The level of performance is
evaluated by the supervisor of the employee and it is one of the factors
that affects the level of bonus.  The performance depends on the education
as well as the attitude ($A$) of the employee and the team.  The attitude
may also have a direct effect on the bonus.

We are interested in estimating the total causal effect of performance on
the salary, i.e.,~quantifying how changes in performance affect the salary. 
The salary ($Y$), bonus ($B$), education ($E$) and performance ($X$) of the
employees are available from the registry of the human resource department. 
Data on attitude ($A$), bonus ($B$) and performance ($X$) have been
collected in an anonymized survey that cannot be linked to the registry on
the personal level.  The question of interest is the identifiability of $P(Y
\mid \doo(X))$ from the data sources $P(Y,B,E,X)$ and $P(A,B,X)$.  Using
the machinery developed in this paper, we can identify the causal effect
with the formula
\[
  P(Y \mid \doo(X)) = \sum_{B,A}P(A)P(B \mid X,A)\sum_{E} P(E)P(Y \mid X,B,E),
\]
where the conditional distributions can be directly determined from the
available data sources (see Section~\ref{sect:new_results} for further
details).

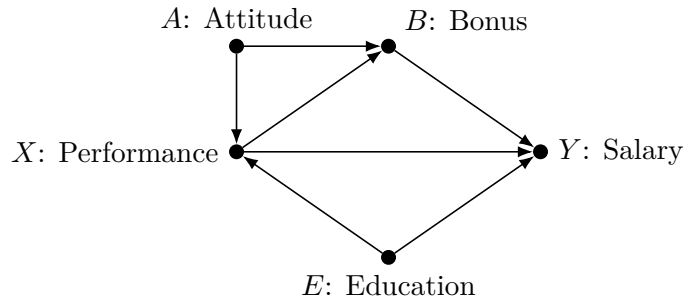
\begin{figure}[t!]
  \centering
\begin{tikzpicture}[scale=2.0]
\node [dot5 = {0}{0}{X}{left}{X \textrm{: Performance}}] at (0,0) {};
\node [dot5 = {0}{0}{E}{below}{E \textrm{: Education}}] at (1,-0.7) {};
\node [dot5 = {0}{0}{Y}{right}{Y \textrm{: Salary}}] at (2,0) {};
\node [dot5 = {0}{0}{B}{above right}{B \textrm{: Bonus}}] at (1,0.7) {};
\node [dot5 = {0}{0}{A}{above}{A \textrm{: Attitude}}] at (0,0.7) {};

\draw [->] (X) -- (Y);
\draw [->] (X) -- (B);
\draw [->] (E) -- (Y);
\draw [->] (E) -- (X);
\draw [->] (B) -- (Y);
\draw [->] (A) -- (X);
\draw [->] (A) -- (B);
\end{tikzpicture}
  \caption{Graph for the example on human resource management in a company. }
  \label{fig:hr}
\end{figure}
To combat the inherent computational complexity of the search-based
approach, we derive rules and techniques that avoid unnecessary
computational steps.  We are able to detect trivial queries where
non-identifiability can be determined directly from the inputs.  We also
present a search heuristic that considerably speeds up the search in cases
where the effect is indeed identifiable.  The approach, called \dosearch{},
is provably sound and it retains the completeness in the cases previously
proven to be solved by the rules of do-calculus.  We can easily scale up to
the problem sizes commonly reported in the literature.  The \proglang{R}~package
\pkg{dosearch} \citep{rsoft,dosearch} provides an implementation of
\dosearch{} and is available from the Comprehensive \proglang{R} Archive
Network (CRAN) at \url{https://CRAN.R-project.org/package=dosearch}.

\begin{table}[t!]
\centering
\begin{footnotesize}
\begin{tabular}{@{\hspace{4pt}} p{0.15cm} @{\hspace{9pt}} p{4.1cm} @{\hspace{4pt}} p{2cm} p{4.2cm} @{\hspace{4pt}} p{1.2cm} @{\hspace{12pt}} p{2.1cm} @{}}
\hline
& & & & \textbf{Missing} &  \\
& \textbf{Problem} & & \textbf{Input} & \textbf{data} &
\textbf{Method}
 \\
 & \textbf{(Reference)}  & \textbf{Target} & \textbf{(assumptions)} & \textbf{pattern} & \textbf{(complete)}  \\ 
\hline
1  & Causal effect identifiability & $P(\+Y \given \doo(\+ X))$ & $P(\+V)$ & None & ID  \\
   & {\scriptsize \citep{Shpitser}} & & & & {\scriptsize (Yes)} \\ [0.1cm] 
2  & Causal effect identifiability  & $P(\+Y \given \doo(\+ X), \+ Z)$ & $P(\+V)$ & None & IDC \\ 
   & {\scriptsize \citep{Shpitser_conditional}} & & & & {\scriptsize (Yes)} \\  [0.1cm]
3  & $z$-identifiability & $P(\+Y \given \doo(\+ X), \+ Z)$    & $P(\+V)$, $P(\+V \setminus \+B \given \doo(\+B))$ & None & zID   \\ 
   & {\scriptsize \citep{Bareinboim:zidentifiability}} & & (NE, ED) & & {\scriptsize (Yes)} \\ [0.1cm]
4 & $g$-identifiability & $P(\+Y \given \doo(\+ X))$     & $\{P(\+V \setminus \+B_i \given \doo(\+ B_i) \}$  & None & gID \\  
    &  {\scriptsize \citep{lee2019general}} & & (ED) & & {\scriptsize (Yes)} \\  [0.1cm]
5  & Surrogate outcome  & $P(\+Y \given \doo(\+ X), \+ Z)$     & $\{P(\+ A_i \given \doo(\+ B_i), \+ C_i) \}$  & None & TRSO  \\ 
   & identifiability & & (NE, SO) & & {\scriptsize (No)}  \\ 
   & {\scriptsize \citep{Tikka:surrogate}} & & & & \\ [0.1cm]
6  & $mz$-transportability & $P(\+Y \given \doo(\+ X), \+ Z)$ & $\{P(\+V \setminus (\+ B_i \cup \+ T_i) \given \doo(\+B_i), \+ T_i)\}$  & None & TR$^{\textrm{mz}}$  \\
   & {\scriptsize \citep{bareinboim2014transportability}} & & (NEDD, ED)& & {\scriptsize (Yes)}  \\ [0.1cm]    
7  & Selection bias recoverability & $P(\+Y \given \doo(\+ X), \+ Z)$ & $P(\+V \setminus S \given S)$ & Selection & RC \\ 
   &  {\scriptsize \citep{bareinboim2015recovering}} & & & & {\scriptsize (Unknown)}   \\  [0.1cm]
\emph{\textbf{8}} & \emph{\textbf{Generalized identifiability}}  & $P(\+Y \given \doo(\+ X), \+ Z)$     & $\{ P(\+ A_i \given \doo(\+ B_i), \+ C_i) \}$  & \emph{\textbf{None}} & \emph{\textbf{do-search}} \\
   & & & & & {\scriptsize (Unknown)} \\ [0.1cm]
9  & Missing data recoverability & $P(\+V)$ & $P(\+ V^*)$ & Restricted & -- \\ 
   & {\scriptsize \citep{Mohan2013}} & & & & {\scriptsize (Yes)}  \\ [0.1cm]
%
%
10  & Missing data recoverability & $P(\+V)$ & $P(\+ V^*)$ & Arbitrary & --  \\ [0.1cm]
   & {\scriptsize \citep{bhattacharya2019}} & & & & {\scriptsize (Unknown)}   \\ [0.1cm]
\emph{\textbf{11}} & \emph{\textbf{Generalized identifiability}}  & $P(\+ Y \given \doo(\+ X), \+ Z)$     & $\{ P(\+ A_i^* \given \doo(\+ B_i), \+ C_i^*) \}$      & \emph{\textbf{Arbitrary}} &  \emph{\textbf{do-search}}   \\ 
& \emph{\textbf{with missing data}} & &  &  & {\scriptsize (No)} \\
\hline
\end{tabular}
\end{footnotesize}
\caption{Solved and unsolved problems in causal effect identification. 
Bold-italic denotes the previously unsolved problems for which \dosearch{}
can now be used.  Input $P(\+ V)$ stands for passively observed joint
distribution of all variables.  Input $P(\+V^*)$ is the joint distribution
with missing data (see Section~\ref{sect:missingness}).  The variable sets
present in the same distribution are disjoint.  Input $P(\+V \setminus \+ B
\mid \doo(\+B))$ stands for an experiment where all variables are measured
and input $P(\+ A \mid \doo(\+ B))$ stands for an experiment where only a
subset $\+ A \subset \+V$ of the variables is measured.  Notation
$\{\cdot\}$ denotes a set of inputs enumerated by the index $i$.  The
assumptions of nested experiments (NE), entire distributions (ED) and nested
experiments in different domains (NEDD) are explained in
Section~\ref{sect:problem}.  Assumptions related to surrogate outcomes (SO)
can be found in \citep{Tikka:surrogate}.  Input $P(\+V \mid S)$ means the
joint distribution under selection bias.  The last column gives the name of
an algorithm that can be used to solve the problem if one exists and whether
it (or a theorem when no algorithm is provided) provides a complete solution
to the problem, or whether the completeness status is not known.  An
algorithm is complete if it returns a correct formula precisely when the
target query is identifiable.  Problems~1--7 are special cases of Problem~8
and Problems~1--10 are special cases of Problem~11.  }
\label{tab:stateofart}
 \end{table}
Other available software for causal effect identifiability problems are only
applicable to a subset of the problems presented in
Table~\ref{tab:stateofart}.  The \pkg{causaleffect} \proglang{R}~package
\citep{Tikka:identifying} can only be used for problems on Rows~1--3 and
5--7 of the table, providing implementations of the relevant algorithms. 
For the problem on Row 1, the generalized adjustment criterion
\citep{perkovic2015} can be applied using the \pkg{dagitty} package
\citep{dagitty} or the \pkg{pcalg} package \citep{pcalg} in \proglang{R}. 
The generalized back-door criterion \citep{MAATHUIS_GBD} is also available
in the \pkg{pcalg} \proglang{R}~package.  The standard back-door criterion
is available in the \proglang{Python} package \pkg{DoWhy} \citep{dowhy}. 
Algorithms implemented in \pkg{causaleffect} run in polynomial time and can
outperform \pkg{dosearch} in their respective \emph{restricted} problem
settings especially with larger graphs.  For a comprehensive performance
comparison between \pkg{causaleffect} and various adjustment criteria, see
\citep{VANDERZANDER20191}.

The paper is structured as follows.  Section~\ref{sect:problem} formulates
our general identification problem and explains the scenarios in
Table~\ref{tab:stateofart} and previous research in detail. 
Section~\ref{sect:search} presents the search algorithm, including the rules
we use, search space reduction techniques, heuristics and theoretical
properties.  Section~\ref{sect:missingness} shows how the search can be
extended to problems that involve missing data.  Section~\ref{sect:package}
demonstrates how the search can be used in \proglang{R} via the
\pkg{dosearch} package.  Efficacy of the search is assessed via simulations. 
Section~\ref{sect:new_results} shows a number of new problems for which we
can find solutions by using the search including a real-world application. 
These problems include combined transportability and selection bias,
multiple sources of selection bias, and causal effect identification from
arbitrary (experimental) distributions.  This section also includes a
systematic analysis of missing data problems and case-control designs. 
Section~\ref{sect:discussion} discusses the merits and limitations of the
approach.  Section~\ref{sect:conclusion} offers concluding remarks.

\section{The general causal effect identification problem} \label{sect:problem}

Our presentation is based on Structural Causal Models (SCM) and the language
of directed graphs.  We assume the reader to be familiar with these concepts
and refer them to detailed works on these topics for extended discussion and
descriptions, such as \citep{Pearl:book2009} and \citep{Koller09}.

Following the standard set-up of do-calculus \citep{pearl1995causal}, we
assume that the causal structure can be represented by a
\emph{semi-Markovian causal graph} $G$ over a set of vertices $\+ V$ (see
Fig~\ref{fig:simple_search}(\subref{fig:intro_backdoor}) for example).  The
directed edges correspond to direct causal relations between the variables
(relative to $\+ V$); directed edges do not form any cycles.  Confounding of
any two observed variables in $\+ V$ by some unobserved common cause is
represented by a bidirected edge between the variables.  This graphical
representation allows us to deal with any causal structures where some
variables are unmeasured.  We assume a positive distribution over the
variables \citep{HUANGVALTORTA} ensuring that all considered causal effects
and conditional distributions are well-defined.

In a non-parametric setting, the problem of expressing a causal quantity of
interest in terms of available information has been described in various
ways depending on the context.  When available data are affected by
selection bias or missing data, a typical goal is to ``recover'' a joint or
marginal distribution.  If data are available from multiple conceptual
domains, a distribution is ``transported'' from the source domains, from
which a combination of both observational and experimental data are
available, to a target domain.  The aforementioned settings can be expressed
in the SCM framework by equipping the graph of the model with special
vertices.  However, on a fundamental level these problems are simply
variations of the original identifiability problem of causal effects and as
such, our goal is to represent them as a single generalized identifiability
problem.  Formally, identifiability can be defined as follows
\citep{Pearl:book2009,shpitser2008}.

\begin{definition}[Identifiability] \label{def:identifiability}
Let $\+ M$ be a set of models with a description $T$ and two objects $\phi$
and $\theta$ computable from each model.  Then $\phi$ is \emph{identifiable}
from $\theta$ in $T$ if $\phi$ is uniquely computable from $\theta$ in any
model $M \in \+ M$.  In other words, all models in $\+ M$ which agree on
$\theta$ also agree on $\phi$.
\end{definition}
In the simplest case, the description $T$ refers to the graph induced by
causal model, $\theta$ is the joint distribution of the observed variables
$P(\+ V)$ and the query $\phi$ is a causal effect $P(Y \mid \doo(X))$.  On
the other hand, proving non-identifiability of  $\phi$ from $\theta$ can be
obtained by describing two models $M^1, M^2 \in \+ M$ such that $\theta$ is
the same in $M^1$ and $M^2$, but object $\phi$ in $M^1$ is different from
$\phi$ in $M^2$.

The general form for a causal identifiability problem that we consider in
this paper is formulated as follows.
\begin{description}
\item[Input:]
A set of input distributions of the form $P(\+ A_i \mid \doo(\+ B_i),\+
C_i)$, a query $P(\+ Y \mid \doo(\+ X), \+ Z)$  and a semi-Markovian
causal graph $G$ over $\+ V$.  \item[Task:]

Output a formula  for the query  $P(\+ Y \mid \doo(\+ X),\+ Z)$ over the
input distributions, or decide that it is not identifiable.

\end{description}
Here  $\+ A_i,\+ B_i, \+ C_i$ are disjoint subsets of $\+ V$ for all $i$,
and $\+ X,\+ Y,\+ Z$ are disjoint subsets of $\+ V$.  The causal graph $G$
may contain vertices which describe mechanisms related to transportability
and selection bias.  In the following sections we explain several important
special cases of this problem definition, some that have been considered in
the literature and some which have not been.

\subsection{Previously considered scenarios as special cases}
\label{sect:prev_scenarios}

We restate the concepts of transportability and selection bias under the
causal inference framework, and show that identifiability in the scenarios
of Rows~1--7 of Table~\ref{tab:stateofart} falls under the general form on
Row~8.  We return to problems that involve missing data on Rows~9--11 later
in Section~\ref{sect:missingness}.

\subsubsection{Causal Effect Identifiability} Input is restricted to a passive
observational distribution $P(\+ V)$.  The target is either a causal effect
$P(\+ Y \mid \doo ( \+ X) )$ for Row~1 of Table~\ref{tab:stateofart} or a
conditional causal effect $P(\+ Y \mid \doo ( \+ X), \+ Z )$ for Row~2 of
Table~\ref{tab:stateofart} \citep{Shpitser,Shpitser_conditional}.

\subsubsection{$z$-identifiability} Similarly to ordinary causal effect
identification, the input consists of the passive observational distribution
$P(\+ V)$ but also of experimental distributions known as surrogate
experiments intervening on a set $\+ B$ \citep{Bareinboim:zidentifiability}. 
Two restricting assumptions, which we call nested experiments and entire
distributions, apply to surrogate experiments.  Experiments are called
nested experiments (NE) when for each experiment intervening a set of
variables $\+B$, experiments intervening on all subsets of $\+B$ are
available as well.  Entire distributions (ED) denote the assumption that the
union of observed and intervened variables is always the set of all
variables $\+V$.

\subsubsection{$g$-identifiability} Unlike $z$-identifiability, the input does
not contain the passive observational distribution $P(\+ V)$ but consists
instead of surrogate experiments on the sets $\{ \+ B_i \}$
\citep{lee2019general} without the assumption of nested experiments.  The
assumption of entire distributions holds.

\subsubsection{Surrogate Outcome Identifiability} Surrogate outcomes generalize
the notion of surrogate experiments from $z$-identifiability.  For surrogate
outcomes, the assumption of nested experiments still holds, but the
assumption of entire distributions can be dropped.  Some less strict
assumptions (SO) still apply \citep{Tikka:surrogate}.  The idea of surrogate
outcomes is that data from previous experiments are available, but the
target $\+ Y$ was at most only partially measured in these experiments and
the experiments do not have to be disjoint from $\+ X$.

\subsubsection{Transportability} The problem of incorporating data from multiple
causal domains is known as transportability \citep{bareinboim2013general}. 
Formally, the goal is to identify a query in a target domain $\pi^*$ using
data from source domains $\pi_1, \ldots ,\pi_n$.  The domains are
represented in the causal graph using a special set of transportability
nodes $\+ T$ which is partitioned into disjoint subsets $\+ T_1,\ldots,\+
T_n$ corresponding to each domain $\pi_i$.  The causal graph contains an
extra edge $T_{ij} \rightarrow V_j$ whenever a functional discrepancy in
$f_{V_j}$ or in $P(u_{V_j})$ exists between the target domain $\pi^*$ and
the source domain $\pi_i$.  The discrepancy is active if $T_{ij} = 1$ and
inactive otherwise.  A distribution associated with a domain $\pi_i$ is of
the form $P(\+ A \mid \doo(\+ B),\+ C,\+ T_i = 1,\+ T_{-i} = 0)$, where
$\+ T_{-i}$ denotes the other subsets of the partition of $\+ T$ except $\+
T_i$.  In other words, only the discrepancies between the $\pi_i$ and
$\pi^*$ are active.  A distribution corresponding to the target domain has
no active discrepancies meaning that it is of the form $P(\+ A \mid
\doo(\+ B),\+ C,\+ T = 0$).  Any variable is conditionally independent from
inactive transportability nodes since their respective edges vanish. 
Furthermore, since transportability nodes set to $0$ vanish, we can assume
any present transportability node to have the value $1$.  Thus an input
distribution from a domain $\pi_i$ takes the form $P(\+ A \mid \doo(\+
B),\+ C,\+ T_i)$.  In the specific case of $mz$-transportability, the
assumptions of entire distributions (ED) and nested experiments in different
domains (NEDD) apply, which means that $P(\+ V \setminus (\+ B_i^\prime \cup
\+ T_i) \mid \doo(\+ B_i^\prime), \+ T_i)$ is available for every subset
$\+ B_i^\prime$ of $\+ B_i$ in each domain $\pi_i$.

\subsubsection{Selection Bias Recoverability} Selection bias can be seen as a
special case of missing data, where the mechanism responsible for the
preferential selection is represented in the causal graph by a special sink
vertex $S$ \citep{Bareinboim2012:selectionbias}.  Typical input for the
recoverability problem is $P(\+ V \mid S = 1)$, the joint distribution
observed under selection bias.  Just as in the case of transportability
nodes, selection bias nodes only appear when the mechanism has been enabled. 
Thus we may assume that the input is of form $P(\+ V \mid S)$.  More
generally, we can consider input distributions of the form $P(\+ A \mid
\doo(\+ B),\+ C, S)$.

\subsection{New scenarios as special cases} \label{sect:new_scenarios}

The following settings are special cases of the general identifiability
problem of Row~8 in Table~\ref{tab:stateofart} that do not fall under any of
the problems of Rows~1--7.  They serve as interesting additions to the cases
considered in the literature.  Concrete examples on these new scenarios are
presented in Section~\ref{sect:new_results}.  Section~\ref{sect:missingness}
extends the general problem of Row~8 in Table~\ref{tab:stateofart} to the
general problem with missing data on Row~11 while also showcasing the
special cases of Rows~9 and 10.

\subsubsection{Multiple Data Sources with Partially Overlapping Variable Sets}

The scenario where only subsets of variables are ever observed together has
been extensively considered in the causal discovery literature
\citep{danks2009integrating,tillman2011learning,triantafillou2010learning},
but not in the context of causal effect identification.  In the basic
setting the input consists of passively observed distributions $P(\+ A_i)$
such that $\+ A_i \subset \+ V$.  We may also observe experimental
distributions $P(\+ A_i \mid \doo (\+ B_i))$
\citep{heh2012uai,overlapping_sofia} or even conditionals $P(\+ A_i \mid
\doo(\+ B_i), \+ C_i)$.  Our approach sets no limitations for the number or
types of input distributions.

\subsubsection{Combining Transportability and Selection Bias} To the best of our
knowledge, the frameworks of transportability and selection bias have not
been considered simultaneously.  The combination of these scenarios fits
into the general problem formulation.  For example, we may have access to
two observational distributions originating from different source domains,
but affected by the same biasing mechanism: $P(\+ A_1 \mid \+ C_1, T_1,
S)$ and $P(\+ A_2 \mid \+ C_2, T_2, S)$, where $T_1$ and $T_2$ are the
transportability nodes corresponding to the two source domains and $S$ is
the selection bias node.

\subsubsection{Recovering from Multiple Sources of Selection Bias} In recent
literature on selection bias as a causal inference problem, the focus has
been on settings where only a single selection bias node is present
\citep[e.g.,][]{Bareinboim2014:selectionbias,Correa2017,Correa2018}. 
However, multiple sources of selection bias are typical in longitudinal
studies where dropout occurs at different stages of the study.  Our approach
is applicable for an arbitrary number of selection bias mechanisms and input
distributions affected by arbitrary combinations of these mechanisms.  In
other words, if $\+ S$ is the set of all selection bias nodes present in the
graph, the inputs can take the form $P(\+ A \mid \doo(\+ B), \+ C, \+
S^\prime)$, where $\+ S^\prime$ is an arbitrary subset of $\+ S$.

\section{A search based approach for causal effect identification}
\label{sect:search}

The key to identification of causal effects is that interventional
expressions can be manipulated using the rules of do-calculus.  We present
these rules for augmented graphs where an additional intervention variable
$I_X$ such that $I_X \rightarrow X$ is added to the induced graph for each
variable $X$ \citep{SGS,Pearl:book2009,lauritzen2000causal} (see
Figure~\ref{fig:simple_search}(\subref{fig:augmented})).
%
%
Now a $d$-separation condition (or $m$-separation~\citep{msep}) of the form $\+
Y \independent{} \+ I_{\+ Z} \mid \+ X, \+ Z, \+W \doublebar \+X$ means
that nodes $\+ Y$ and intervention nodes  $\+ I_{\+ Z}$ of $\+ Z$ are
$d$-separated given $\+ X$, $\+ Z$ and $\+ W$ in a graph where edges incoming
to (intervened) $\+ X$ have been removed
\citep{hyttinen2015,dawid2002influence}.  The three rules of do-calculus
\citep{pearl1995causal} can be expressed as follows:
\begin{equation} \label{eq:docalc}
\begin{aligned}
P(\+ Y \mid \doo(\+ X), \+ Z, \+ W) &= P(\+ Y \mid \doo(\+ X), \+ W),
\text{ if } \+ Y \independent{} \+ Z \mid \+ X, \+ W \doublebar \+ X \\
P(\+ Y \mid \doo(\+ X, \+ Z), \+ W) &= P(\+ Y \mid \doo(\+ X), \+ Z, \+
W), \text{ if } \+ Y \independent{} \+ I_{\+ Z} \mid \+ X, \+ Z, \+W
\doublebar \+X \\
P(\+ Y \mid \doo(\+ X, \+ Z), \+ W) &= P(\+ Y \mid \doo(\+ X), \+ W),
\text{ if } \+ Y \independent{}\+  I_{\+ Z} \mid \+ X, \+W \doublebar \+X
\end{aligned}
\end{equation}
The rules are often referred to as insertion/deletion of observations,
exchange of actions and observations, and insertion/deletion of actions
respectively.  Each rule of do-calculus is only applicable if the
accompanying $d$-separation criterion (on the right-hand side) holds in the
underlying graph.  In addition to these rules, most derivations require
basic probability calculus.
\begin{algorithm}[!t]
  \begin{algorithmic}[1]
  \INPUT{Target $Q = P(\+ Y \mid \doo(\+ X),\+ W)$, a semi-Markovian graph
  $G$ and a set of known input distributions $\+ P=\{P_1,\ldots,P_n\}$}.
  \OUTPUT{A formula for $Q$ 
  or NA if the effect is not identifiable.}
  \State \textbf{for each} $P_i \in \+ P$ \textbf{do}
    \State \quad Derive new distributions 
     from $P_i$ 
     such that:
    \Statex \quad \quad \textbullet~The required $d$-separation criteria are satisfied by $G$.
    \Statex \quad \quad \textbullet~Any possible additional input required must also be in $\+ P$.
  \State \quad Add the new identified distributions to $\+ P$.
  \State \quad If $Q$ was derived, return a formula for it.
  \State Return NA.
  \end{algorithmic}
  \caption{An outline of a search for causal effect identification. }
  \label{prog:naive}
\end{algorithm}
Do-calculus directly motivates a forwards search over its rules.  The
outline of this type of search is given in Algorithm~\ref{prog:naive}.  The
algorithm derives new identifiable distributions based on what has been
given as the input or identified in the previous steps.  For each identified
distribution every rule of do-calculus and standard probability
manipulations of marginalization and conditioning are applied in succession,
until the target distribution is found, or no new distributions can be found
to be identifiable.  A preliminary version of this kind of search is used by
\citet{hyttinen2015} as a part of an algorithmic solution to causal effect
identifiability when the underlying graph is unavailable.

The formulas produced by Algorithm~\ref{prog:naive} correspond to short
derivations and unnecessarily complicated expressions are avoided.  Also,
only distributions guaranteed to be identifiable are derived and used during
the search.  Formulas for intermediary queries that were identified during
the search are also available as a result.  Alternatively, one could also
start with the target and search towards the input distributions; a search
in this direction will spend time deriving a number expressions that are
inevitably non-identifiable based on the input.  A depth-first search would
produce unnecessarily complicated expressions.
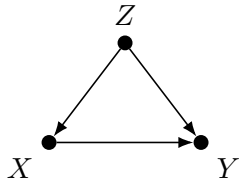
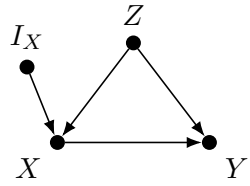
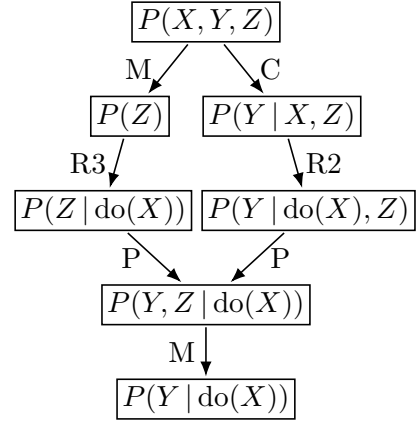
\begin{figure}[t!]
  \centering
  \begin{subfigure}[t]{0.3\textwidth}
    \centering
    \begin{tikzpicture}[scale=2]
    \node [dot = {0}{0}{X}{below left}] at (0,0) {};
    \node [dot = {0}{0}{Y}{below right}] at (1,0) {};
    \node [dot = {0}{0}{Z}{above}] at (0.5,0.66) {};

    \draw [->] (X) -- (Y);
    \draw [->] (Z) -- (Y);
    \draw [->] (Z) -- (X);
    \end{tikzpicture}
    \caption{Example graph.}
    \label{fig:intro_backdoor}
  \end{subfigure}
  \hfill
    \begin{subfigure}[t]{0.25\textwidth}
    \centering
    \begin{tikzpicture}[scale=2]
    \node [dot = {0}{0}{I_X}{above}] at (-0.2,0.5) {};
    \node [dot = {0}{0}{X}{below left}] at (0,0) {};
    \node [dot = {0}{0}{Y}{below right}] at (1,0) {};
    \node [dot = {0}{0}{Z}{above}] at (0.5,0.66) {};

    \draw [->] (I_X) -- (X);
    \draw [->] (X) -- (Y);
    \draw [->] (Z) -- (Y);
    \draw [->] (Z) -- (X);
    \end{tikzpicture}
    \caption{Augmented graph.}
    \label{fig:augmented}
  \end{subfigure}
  \hfill
  \begin{subfigure}[t]{0.4\textwidth}
    \centering
    \begin{tikzpicture}{scale=0.82}
    \node[bordered] (v1) at (0,4) {$P(X,Y,Z)$};
    \node[bordered] (v2) at (-1,2.75) {$P(Z)$};
    \node[bordered] (v3) at (1,2.75) {$P(Y \mid X,Z)$};
    \node[bordered] (v4) at (-1.35,1.5) {$P(Z \mid \doo(X))$};
    \node[bordered] (v5) at (1.35,1.5) {$P(Y \mid \doo(X), Z)$};
    \node[bordered] (v6) at (0,0.25) {$P(Y,Z \mid \doo(X))$};
    \node[bordered] (v7) at (0,-1) {$P(Y \mid \doo(X))$};
    \draw[->] (v1) edge node[left,xshift=-2] {M} (v2);
    \draw[->] (v1) edge node[right,xshift=2] {C} (v3);
    \draw[->] (v2) edge node[left] {R3} (v4);
    \draw[->] (v3) edge node[right] {R2} (v5);
    \draw[->] (v4) edge node[left,xshift=-1] {P} (v6);
    \draw[->] (v5) edge node[right,xshift=1] {P} (v6);
    \draw[->] (v6) edge node[left] {M}  (v7);
    \end{tikzpicture}
    \caption{A derivation for $P(Y \mid \doo(X))$.}
    \label{fig:derivation}
  \end{subfigure}
  \caption{
  The back-door criterion holds in the example graph
   (\subref{fig:intro_backdoor}) for $Z$.  The augmented graph
   (\subref{fig:augmented}) includes the intervention node $I_X$ for $X$
   explicitly.  The labels M, C, P, R2 and R3 in the derivation of
   (\subref{fig:derivation}) refer to marginalization, conditioning, product
   rule and Rules~2 and 3 of do-calculus respectively (see
   Table~\ref{tab:rules}).  The required $d$-separation conditions $Y
   \independent I_X \mid Z,X$ for R2 and $Z \independent I_X$ for R3 hold
   in the augmented graph (\subref{fig:augmented}).  }
  \label{fig:simple_search}
\end{figure}
The search can easily derive for example the back-door criterion in the
graph of Figure~\ref{fig:simple_search}(\subref{fig:intro_backdoor}) as
shown by the derivation in
Figure~\ref{fig:simple_search}(\subref{fig:derivation}).  The target is $Q =
P(Y \mid \doo(X))$ and input is $\+ P = \{P(X,Y,Z)\}$.  From $P(X,Y,Z)$
the search first derives the marginal $P(Z)$ and the conditional $P(Y \mid
X,Z)$.  Then $P(Z \mid \doo(X))$ is derived by the third rule of
do-calculus because $Z \independent I_X$.  The second rule derives $P(Y
\mid \doo(X), Z)$  from $P(Y \mid X,Z)$ as $Y \independent I_X \mid
Z,X$.  The two terms can be combined via the product rule of probability
calculus to get $P(Y, Z \mid \doo(X))$ and finally the target is $P(Y
\mid \doo(X))$ is just a marginalization of this.  The familiar formula
$\sum_{Z} P(Y \mid X,Z)P(Z)$ is thus obtained.

However, it is not straightforward to make a search over do-calculus
computationally feasible.  The search space in
Figure~\ref{fig:simple_search}(\subref{fig:derivation}) shows only the parts
that resulted in the identifying formula: for example all passively observed
marginals and conditionals over $\+ V$ can be derived from the input $P(\+
V)$.  Especially in a non-identifiable case a naive search may go through a
huge space before it can return the non-identifiable verdict.  The choice of
rules is also not obvious: a redundant rule may make the search faster or
slower; false non-identifiability may be concluded if a necessary rule is
missing.  Also the order in which the rules are applied can have a large
impact on the performance of the search.  In the following sections we will
provide non-trivial solutions to these challenges.

\subsection{Rules} \label{sect:rules}

Table~\ref{tab:rules} lists the full set of rules used to manipulate
distributions during the search, generalizing the work by
\citet{hyttinen2015}.
\begin{table}[t!]
\centering
\begin{small}
\begin{tabular}{lllll}
\hline
 Rule & Additional Input & Output & Description \\
\hline
 $1+$ && $P(\+ Y \mid \doo(\+ X), \+ Z, \+ W)$ & Insertion of observations \\
 $1-$ && $P(\+ Y \mid \doo(\+ X), \+ W \setminus \+ Z)$ & Deletion of observations \\
 $2+$ && $P(\+ Y \mid \doo(\+ X,\+ Z), \+ W \setminus \+ Z)$ & Observation to action exchange \\
 $2-$ && $P(\+ Y \mid \doo(\+ X \setminus \+ Z),\+ Z, \+ W)$ & Action to observation exchange \\
 $3+$ && $P(\+ Y \mid \doo(\+ X,\+ Z), \+ W)$ & Insertion of actions \\
 $3-$ && $P(\+ Y \mid \doo(\+ X \setminus \+ Z), \+ W)$ & Deletion of actions \\
\hline
 $4$ && $P(\+ Y \setminus \+ Z \mid \doo(\+ X),\+ W)$ & Marginalization \\
 $5$ && $P(\+ Y \setminus \+ Z \mid \doo(\+ X),\+ Z,\+ W)$ & Conditioning \\
 $6+$ & $P(\+ Z \mid \doo(\+X),\+ W \setminus \+ Z)$ & $P(\+Y, \+Z \mid \doo(\+ X),\+ W \setminus \+ Z)$ & Chain rule multiplication \\
 $6-$ & $P(\+ Z \mid \doo(\+ X), \+ Y,\+ W)$  & $P(\+Y, \+Z \mid \doo(\+ X),\+W)$ & Chain rule multiplication \\
\hline
\end{tabular}
\end{small}
\caption{The rules used to manipulate input distributions of the form $P(\+
Y \mid \doo(\+ X), \+ W)$.  The output distribution is identified if the
input has been previously identified and if the corresponding $d$-separation
Criteria~\ref{eq:docalc} hold in the graph (for Rules~$1\pm,2\pm$ and
$3\pm$) or if the additional input has also been identified (Rules~$6\pm$). 
The sets $\+ Y, \+ X$ and $\+ W$ are disjoint.  The role of the set $\+ Z$
depends on the rule being applied (see Table~\ref{tab:validsets}).}
\label{tab:rules}
\end{table}
\noindent
\subsubsection{Do-calculus} Rules $1\pm, 2\pm$ and $3\pm$ correspond to the
rules of do-calculus such that Rules~$1+$, $2+$ and $3+$ are used to add
conditional variables and interventions and Rules~$1-, 2-, 3-$ are used to
remove them.  Each rule is only valid if the corresponding $d$-separation
criterion given in the beginning of Section~\ref{sect:search} holds.

\subsubsection{Probability theory} Rule~$4$ performs marginalization over $\+ Z
\subset \+ Y$, and produces a summation at the formula level:
\[
 P(\+Y \setminus \+ Z \mid \doo(\+ X),\+ W) = \sum_{\+ Z} P(\+ Y \mid
 \doo(\+ X),\+ W).
\]
Similarly, Rule~$5$ conditions on a subset $\+ Z \subset \+ Y$ to obtain the
following formula:
\[
  P(\+ Y \setminus \+ Z \mid \doo(\+ X),\+Z, \+ W) = \frac{P(\+ Y \mid
  \doo(\+ X),\+ W)}{\sum_{\+ Y \setminus \+ Z} P(\+ Y \mid \doo(\+ X),\+
  W)}.
\]
Rules $6+$ and $6-$ perform multiplication using the chain rule of
probability which requires two known distributions.  When Rule~$6+$ is
applied, the distribution $P(\+ Y \mid \doo(\+ X),\+ W)$ is known and we
check whether $P(\+ Z \mid \doo(\+ X),\+ W \setminus \+ Z)$ is known as
well.  For Rule~$6-$, the roles of the distributions are reversed.  In the
case of Rule~$6+$, $\+ Z$ is a subset of $\+ W$ and we obtain
\[
  P(\+ Y,\+ Z \mid \doo(\+ X), \+ W \setminus \+ Z) = P(\+ Y \mid
  \doo(\+ X),\+ W) P(\+ Z \mid \doo(\+ X),\+ W \setminus \+ Z).
\]
The two version of the chain rule are needed: it may be the case that when
expanding $P(\+ Y \mid \doo(\+ X), \+ W)$ with Rule~$6+$ the additional
input $P(\+ Z \mid \doo(\+X),\+ W \setminus \+ Z)$ is only identified
later in the search.  Then, $P(\+Y, \+Z \mid \doo(\+ X),\+W)$ is
identified when Rule~$6-$ is applied to $P(\+ Y \mid \doo(\+X),\+ W)$.

\subsection{Improving the efficacy of the search} \label{sect:improvements}

In this section, we present various techniques that improved the efficiency
of the search.  These findings are implemented in a search algorithm in
Section~\ref{sect:algo}.

\subsubsection{Term expansion}

Term expansion refers to the process of deriving new distributions from an
input distribution using the rules of Table~\ref{tab:rules}.  By \emph{term}
we mean a single identified distribution.  A term is considered
\emph{expanded} if the rules of Table~\ref{tab:rules} have been applied to
it in every possible way when the term is in the role of the input.  Note
that an expanded distribution may still take the role of an additional input
when another term is being expanded.  Consider the step of expanding the
input term in Table~\ref{tab:rules} to all possible outputs with any rule. 
This can be done by enumerating every non-empty subset $\+ Z$ of $\+ V$, and
applying the rule with regard to it.
\begin{table}[t!]
\centering
\begin{small}
\begin{tabular}{llll}
\hline
Rule & Validity condition & Termination condition \\
\hline
$1+$  & $\+ Z \cap (\+ Y \cup \+ X \cup \+ W) = \emptyset$  & \\
$1-$  & $\+ Z \subseteq \+ W$                               & $\+ W = \emptyset$ \\
$2+$  & $\+ Z \subseteq \+ W$                               & $\+ W = \emptyset$ \\
$2-$  & $\+ Z \subseteq \+ X$                               & $\+ X = \emptyset$ \\
$3+$  & $\+ Z \cap (\+ Y \cup \+ X \cup \+ W) = \emptyset$  & \\
$3-$  & $\+ Z \subseteq \+ X$                               & $\+ X = \emptyset$ \\ \hline
$4$   & $\+ Z \subset \+ Y$                                 & $\lvert \+ Y\rvert = 1$ \\
$5$   & $\+ Z \subset \+ Y$                                 & $\lvert\+ Y\rvert = 1$ \\
$6+$  & $\+ Z \subseteq \+ W$                               & $\+ W = \emptyset$ \\
$6-$  & $\+ Z \cap (\+ Y \cup \+ X \cup \+ W) = \emptyset$  & \\ \hline
\end{tabular}
\end{small}
\caption{The conditions for the enumerated subset $\+ Z$ for applying the
rules of Table~\ref{tab:rules} to a term $P(\+ Y \mid \doo(\+ X),\+ W)$. 
For Rules~$6+$ and $6-$, the conditions specify valid variables of the
second required term.}
\label{tab:validsets}
\end{table}
Table~\ref{tab:validsets} outlines the requirements for $\+ Z$ for each rule
of the search.  Table~\ref{tab:validsets} tells us that when an observation
$\+ Z$ is added using Rule~$1+$, it cannot be contained in any of the sets
$\+ Y, \+ X$ or $\+ W$ since they are already present in the term.  Only
observations that are present can be removed, which is why $\+ Z$ has to a
subset of $\+ W$ when applying Rule~$1-$.  We may skip the application of
this rule if the set of observations is empty for the current term.  The
exchange of observations to experiments using Rule~$2+$ has similar
requirements for set $\+ Z$ as Rule~$1-$.  Exchanging experiments to
observations using Rule~$2-$ works in a similar fashion.  Only experiments
that are present can be exchanged which means that $\+ Z \subseteq \+ X$. 
This rule can be skipped if the set of experiments is empty.  New
experiments are added using Rule~$3+$ with similar requirements as Rule~$1+$.
Well-defined subsets for using Rule~$3-$ are the same as for rule~$2-$.
For Rules~$4$ and $5$, the only requirement is that $\+ Z$ is a
proper subset of $\+ Y$.  When the chain rule is applied with Rule~$6+$, we
require that the variables of the second product term is observed in the
first term.  When applied in reverse with Rule~$6-$, the variables of the
second term must not be present in the first term.

\subsubsection{Termination conditions}

Additionally, Table~\ref{tab:validsets} lists the termination condition for
each rule: if it is satisfied by the current term to be expanded we know
that the rule cannot be applied to it.  The following simple lemma shows
that when any of the termination conditions hold, no new distributions can
be derived from it using the respective rule, which allows the search to
directly proceed to the next rule.
\begin{lemma} \label{lem:termination} Let $G$ be a semi-Markovian graph and
let $\+ Y, \+ X$ and $\+ W$ be disjoint subsets of $\+ V$.  Then all of the
following are true:
\begin{itemize}
\item[(i)]{If $\+ W = \emptyset$, then Rule~$1-$ of Table~\ref{tab:rules} cannot be used.} 
\item[(ii)]{If $\+ W = \emptyset$, then Rule~$2+$ of Table~\ref{tab:rules} cannot be used.}
\item[(iii)]{If $\+ X = \emptyset$, then Rule~$2-$ of Table~\ref{tab:rules} cannot be used.}
\item[(iv)]{If $\+ X = \emptyset$, then Rule~$3-$ of Table~\ref{tab:rules} cannot be used.}
\item[(v)]{If $\lvert\+ Y\rvert = 1$, then Rule~$4$ of Table~\ref{tab:rules} cannot be used.}
\item[(vi)]{If $\lvert\+ Y\rvert = 1$, then Rule~$5$ of Table~\ref{tab:rules} cannot be used.}
\item[(vii)]{If $\+ W = \emptyset$, then Rule~$6+$ of Table~\ref{tab:rules} cannot be used. }
\end{itemize}
\end{lemma}
\begin{proof} For (i), the set $\+ W$ is empty so the application of Rule~$1-$
using any subset $\+ Z$ would result in $P(\+ Y \mid \doo(\+ X), \+ W
\setminus \+ Z) = P(\+ Y \mid \doo(\+ X), \+ W)$ which is already
identified.  For (ii), the set $\+ W$ is empty so no observation can be
exchanged for an action using the second rule of do-calculus.  For (iii),
the set $\+ X$ is empty so no action can be exchanged for an observation
using the second rule of do-calculus.  For (iv), the set $\+ X$ is empty so
the application of Rule~$3-$ using any subset $\+ Z$ would result in $P(\+ Y
\mid \doo(\+ X \setminus \+ Z), \+ W) = P(\+ Y \mid \doo(\+ X), \+ W)$
which is already identified.  For (v) and (vi), the set $\+ Y$ only has a
single vertex, so it cannot have a non-empty subset.  For (vii), the set $\+
W$ is empty so no subset $\+ Z \subset \+ W$ can exist for the second input.
\end{proof}

\subsubsection{Rule necessity}

The Rule~1 of do-calculus can be omitted as shown by \citet[Lemma
4]{huangvaltorta:complete}.  Instead of inserting an observation using Rule~1,
we can insert an intervention and then exchange it for an observation. 
Similarly, an observation can be removed by first exchanging it for an
intervention and then deleting the intervention.  It follows that Rules~$1+$
and $1-$ of Table~\ref{tab:rules} are unnecessary for the search.
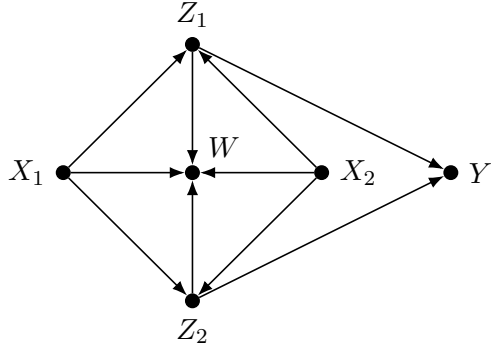
\begin{figure}[t!]
  \centering
  \begin{tikzpicture}[scale=1.7]
  \node [dot = {0}{0}{X_1}{left}] at (0,0) {};
  \node [dot = {0}{0}{X_2}{right}] at (2,0) {};
  \node [dot = {0}{0}{Z_1}{above}] at (1,1) {};
  \node [dot = {0}{0}{Z_2}{below}] at (1,-1) {};
  \node [dot = {0}{0}{W}{above right}] at (1,0) {};
  \node [dot = {0}{0}{Y}{right}] at (3,0) {};

  \draw [->] (X_1) -- (Z_1);
  \draw [->] (X_1) -- (Z_2);
  \draw [->] (X_2) -- (Z_1);
  \draw [->] (X_2) -- (Z_2);
  \draw [->] (Z_1) -- (Y);
  \draw [->] (Z_2) -- (Y);
  \draw [->] (Z_1) -- (W);
  \draw [->] (Z_2) -- (W);
  \draw [->] (X_1) -- (W);
  \draw [->] (X_2) -- (W);
  \end{tikzpicture}
  \caption{A graph for the example where all rules of Table~\ref{tab:rules}
  are required for identifying the target quantity.}
  \label{fig:necessary_rules}
\end{figure}
The following example shows that the remaining rules of
Table~\ref{tab:rules} are all necessary.  In the graph of
Figure~\ref{fig:necessary_rules}, the causal effect $P(Y, X_1 \mid
\doo(X_2), W)$ can be identified from the inputs $P(W \mid \doo(X_2), Y,
X_1)$,  $P(Y \mid \doo(X_2), Z_1, Z_2, X_1)$, $P(X_1 \mid \doo(X_2),
W)$, $P(Z_2, X_2 \mid \doo(X_1))$ and $P(Z_1 \mid \doo(X_1, Y), X_2)$
when all rules are available, but not when any individual rule is omitted. 
This can be verified by running the search algorithm presented at the
beginning of Section~\ref{sect:search} or the more advanced algorithm of
Section~\ref{sect:algo} with each rule switched off individually.

\subsubsection{Early detection of non-identifiable instances}

Worst-case performance of the search can be improved by detecting
non-identifiable quantities directly based on the set of inputs before
launching the search.  The following theorem provides a sufficient criterion
for non-identifiability.
\begin{theorem} \label{thm:nonidsimple} Let $G$ be a semi-Markovian graph,
let $Q = P(\+ Y \mid \doo(\+ X ), \+ W)$ and let
\[
 \+ P = \{P(\+ A_1 \mid \doo(\+ B_1), \+ C_1), \ldots, P(\+ A_n \mid
 \doo(\+ B_n), \+ C_n) \}.
\]
Then $Q$ is not identifiable from $\+ P$ in $G$ via rules of
Table~\ref{tab:rules} if
 \[
 \+ Y \not\subseteq \bigcup_{i = 1}^n \+ A_i,
 \]
\end{theorem}
\begin{proof}
Since $ \+ Y \not\subseteq \bigcup_{i = 1}^n \+ A_i$, there exists a
variable $Y_j \in \+ Y$ such that none of the sets $\+ A_i$ contain it.  No
rule of Table~\ref{tab:rules} outputs a distribution $P(\+ Y^\prime \mid
\doo(\+ X^\prime), \+ W^\prime)$ such that some member of $\+ Y^\prime$
would not already exist on the left-hand side of the input or additional
input of the rule.  Thus there is no sequence of rules that when applied to
the available inputs $\+ P$ would result in a distribution of the form
$P(Y_j, \cdot \mid \doo(\cdot), \cdot)$.  Thus there is no such sequence
for $P(\+ Y \mid \doo(\+ X), \+ W)$.
\end{proof}
In other words, Theorem~\ref{thm:nonidsimple} can be used to verify that the
entire set $\+ Y$ of a target distribution $P(\+ Y \mid \cdot)$ cannot be
constructed from the inputs.  If this is the case, the target quantity is
not identifiable.

\subsubsection{Heuristics}

During the search, we always expand one term at a time through the rules and
store the newly identified distributions.  In order for the search to
perform fast, we need to decide which branches are the most promising and
should therefore be expanded first.  We can do this by defining a proximity
function relating the source terms and the target query, and by always
expanding the closest term first.

Our suggestion here is motivated by the way an educated person might apply
do-calculus in a manual derivation.  Our chosen proximity function $h$ links
the target distribution $P^{t} = P(\+ A_t \mid \doo(\+ B_t),\+C_t)$ and a
source distribution $P^{s} = P(\+ A_s \mid \doo(\+B_s),\+ C_s)$ in the
following way:
\begin{align*}
h(P^t,P^s) &= 10 \lvert \+ A_t \cap \+ A_s \rvert +
              5 \lvert \+ B_t \cap \+ B_s \rvert +
              3 \lvert \+ C_t \cap \+ C_s \rvert -
              2 \lvert \+ A_t \setminus \+ A_s \rvert -
              2 \lvert \+ B_t \setminus \+ B_s \rvert \\
             &\quad - 2
             \lvert\+ B_s \setminus \+ B_t \rvert -
             \lvert\+ C_t \setminus \+ C_s\rvert -
             \lvert\+ C_s \setminus \+ C_t\rvert.
\end{align*}
Each input distribution and terms derived using the search are assigned into
a priority queue, where the priority is determined by the value given by
$h$.  Distributions closer to the target are prioritized over other terms.

The weight 10 for the term $|\+ A_t \cap \+ A_s|$ indicates that having the
correct response variables is considered as the first priority.  Having the
correct intervention is considered as the second priority (weight 5) and
having the correct condition as the third priority (weight 3).  The
remaining terms in $h$ penalize variables that are in the target
distribution but not in the source distribution or vice versa.  Again,
variables that are intervened on are considered to be more important than
conditioning variables.

\subsection{The search algorithm}\label{sect:algo}

We take Algorithm~\ref{prog:naive} as our starting point and compile the
results of Section~\ref{sect:improvements} into a new search algorithm
called \dosearch{}.  This algorithm is capable of solving generalized
identifiability problems (Row~8 in Table~\ref{tab:stateofart}) while
streamlining the search process through a heuristic search order and
elimination of redundant rules and subsets.  The pseudo-code for \dosearch{}
is shown in Algorithm~\ref{prog:do-search}.

The algorithm begins by checking whether the query can be solved trivially
without performing the search.  This can happen if the target $Q$ is a
member of the set of inputs or if Theorem~\ref{thm:nonidsimple} applies. 
Next, we note that each input distribution in the set $\+ P$ is marked as
unexpanded at the beginning of the search.  Distributions in $\+ P$ are
expanded one at a time by applying every rule of Table~\ref{tab:rules} in
every possible way.
\begin{algorithm}[!t]
  \begin{algorithmic}[1]
  \INPUT{Target $Q = P(\+ Y \mid \doo(\+ X),\+ W)$, a semi-Markovian graph $G$ and a set of known distributions $\+ P=\{P_1,\ldots,P_n\}$. }
  \OUTPUT{A formula $F$ for $Q$ in terms of $\+ P$ or NA}
  \State \textbf{if} $Q \in \+ P$, \textbf{return} $Q$
  \State \textbf{if} target is non-identifiable by Theorem~\ref{thm:nonidsimple}, \textbf{then return} NA
  \State \textbf{let} $\+ U$ be the set of unexpanded distributions, 
  initially $\+ U := \+ P$
  \State \textbf{while} $\+ U \neq \emptyset$, \textbf{do}
    \State \quad \textbf{let} $P^\prime$ be the unexpanded distribution closest to the target: $P^\prime = \underset{P_i \in \+ U}{\mathrm{argmax}}\;h(Q, P_i)$
    \State \quad \textbf{let} $\+ M$ be the set of rules of Table~\ref{tab:rules}, without Rules~$1\pm$, such that the termination 
    \Statex \quad \quad conditions of Table~\ref{tab:validsets} do not hold with respect to $P^\prime$.
    \State \quad \textbf{let} $\+ P^*$ be the set of all distributions derived from $P^\prime$ using the rules in $\+ M$
    \State \quad \textbf{for} each new candidate distribution $P^* \in \+ P^*$, \textbf{do}
      \State \quad \quad \textbf{if} $P^*$ is already in $\+ P$, \textbf{then continue}
      \State \quad \quad \textbf{if} the validity conditions of Table~\ref{tab:validsets} are not satisfied by $P^*$, \textbf{then continue}
      \State \quad \quad \textbf{if} an additional input is required that is not in $\+ P$, \textbf{then continue}  
      \State \quad \quad \textbf{if} Rule~$2\pm$ or $3\pm$ of Table~\ref{tab:rules}
      is applied and the corresponding $d$-separation Criterion~\ref{eq:docalc} 
      \Statex \quad \quad \quad is not satisfied by $G$, \textbf{then continue}
      \State \quad \quad \textbf{if} $P^* = Q$, \textbf{then}
      \State \quad \quad \quad Derive a formula $F$ for $Q$ by backtracking.
      \State \quad \quad \quad \textbf{return} $F$
      \State \quad \quad Add $P^*$ to $\+ P$, add $P^*$ to $\+ U$
    \State \quad Mark $P^\prime$ as expanded: remove $P^\prime$ from $\+ U$
  \State \textbf{return} NA
  \end{algorithmic}
  \caption{\dosearch{}}
  \label{prog:do-search}
\end{algorithm}

The iteration over the unexpanded distributions $\+ U$ proceeds as follows
(Lines~4--5).  Each input distribution and terms derived from it are
assigned into a priority queue, where the priority is determined by the
value given by the proximity function $h$.  Distributions closest to the
target are expanded first.  In the implementation, only the actual memory
addresses of the distribution objects are placed into the queue.  The set
$\+ P$ is implemented as a hash table that serves as a container for all
input distributions and those derived from them.  Each new distribution is
assigned a unique index that also serves as the hash function for this
table.  The distribution objects contained in the table are represented
uniquely by three integers corresponding to the sets $\+ A, \+ B$ and $\+ C$
of the general form $P(\+ A \mid \doo(\+ B), \+ C)$.  A distribution
object also contains additional auxiliary information such as which rule was
used to derive it, whether it is expanded or not and from which distribution
it was obtained.  This information is used to construct the derivation if
the target is found to be identifiable.

Multiple distributions can share the same value of the proximity function
$h$.  In the case that multiple candidates share the maximal value, the one
that was derived the earliest takes precedence.  When the unexpanded
distribution currently closest to the target is determined, the rules of
Table~\ref{tab:rules} are applied sequentially for all valid subsets
dictated by Table~\ref{tab:validsets}.  When rules two and three of
do-calculus are considered, the necessary $d$-separation criteria is checked
from $G$ (Line~12).  For the chain rule, the presence of the required second
input is also verified.  The reverse lookup is implemented by using another
hash table, where the hash function is based on the unique representation of
each distribution object.  The values contained in the table are the indices
of the derived distributions.  The same hash table is also used to ensure
that we do not attempt to derive distributions again that have been
previously found to be identifiable from the inputs.

We construct a set $\+ M$ of applicable rules for each unexpanded
distribution $P^\prime$ using the termination conditions of
Table~\ref{tab:validsets} (Line~6).  If all the necessary conditions have
been found to hold for an applicable rule and a subset, the newly derived
distribution $P^*$ is added to the set of known distributions and placed
into the priority queue as an unexpanded distribution.  When the applicable
rules and subsets have been exhausted for the current distribution
$P^\prime$, the term is marked as expanded and removed from the queue (Line~17).
If the target distribution is found at any point (Line~13), a formula
is returned for it in terms of the original inputs.  Alternatively, we can
also continue deriving distributions to obtain different search paths to the
target that can possibly produce different formulas for it.  If instead we
exhaust the set of unexpanded distributions by emptying the queue, the
target is deemed non-identifiable by the search (Line~18).

We keep track of the rules that were used to derive each new distribution in
the search.  This allows us to construct a directed graph of the derivation
where each root node is a member of the original input set $\+ P$ and their
descendants are the distributions derived from them during the search.  Each
edge represents a manipulation of the parent node(s) to obtain the child
node.  For an identifiable target quantity, the formula $F$ is obtained by
backtracking the chain of manipulations recursively until the roots are
reached (Line~14).  The derivation of the example in the beginning of
Section~\ref{sect:search} depicted in
Figure~\ref{fig:simple_search}(\subref{fig:derivation}) can be efficiently
found by applying this procedure.

%
We assess the worst case complexity of \dosearch{} in terms of the input
graph size, which is the primary determining factor of the search time. 
Checks for separation and termination and validation conditions all run in
polynomial time with respect to the number of vertices, but the number of
distributions grows very rapidly.  In a hypothetical absolute worst case
scenario, the target is not identifiable, but every other distribution is. 
Supposing a graph with $d$~vertices, we can determine how many distributions
would have to be derived by the search in order to exhaust the search space. 
For a distribution of the form $P(\+ A \mid \doo(\+ B), \+ C)$, any
variable in $\+ V$ can either belong to the sets $\+ A, \+ B$ or $\+ C$ or
not be present in the distribution, meaning that there are $4^d$ ways to
categorize every variable.  However, we must account for the fact that there
must always be at least one variable in the set $\+ A$.  There are $3^d$
ways to assign all the variables in such a way that no variable is a member
of set $\+ A$.
%
%
Thus the total number of distributions considered by \dosearch{} grows as
$O(4^d)$.  See the simulations of Section~\ref{sec:simu} for running time
performance in more realistic settings.

\subsection{Soundness and completeness properties}

We are ready to establish some key theoretical properties of \dosearch{}. 
The first theorem considers the correctness of the search.
\begin{theorem}[Soundness] \label{thm:soundness} \emph{\dosearch{}} always
terminates: if it returns an expression for the target $Q$, it is correct,
if it returns NA then $Q$ is not identifiable with respect to the rules of
do-calculus and standard probability manipulations (in
Table~\ref{tab:rules}).
\end{theorem}
\begin{proof}
Each new distribution is derived by using only well-defined manipulations as
outlined by Table~\ref{tab:validsets} and by ensuring that the required
separation criteria hold in $G$ when rules of do-calculus are concerned.  It
follows that if the search terminates and returns a formula for the target
distribution, it was reached from the set input distributions through a
sequence of valid manipulations.  If \dosearch{} terminates as a result of
Theorem~\ref{thm:nonidsimple}, we are done.  Suppose now that
Theorem~\ref{thm:nonidsimple} does not apply.  By definition, \dosearch{}
enumerates every rule of Table~\ref{tab:rules} for every well-defined subset
of Table~\ref{tab:validsets}.  By Lemma~\ref{lem:termination}, no
distributions are left out by applying the termination criteria of
Table~\ref{tab:validsets}.  We know that if Rules~$1\pm$ of
Table~\ref{tab:validsets} are omitted, the distributions generated by these
rules can be obtained by a combination of Rules~$2\pm$ and $3\pm$. 
Furthermore, the order in which the distributions are expanded does not
matter, as every possible manipulation is carried out nonetheless.  The
search will eventually terminate, since distributions that have already been
derived are not added again to the set of unexpanded distributions and there
are only finitely many ways to apply the rules of Table~\ref{tab:rules}.
\end{proof}
The following theorem provides a completeness result in connection to
existing identifiability results.  Since do-calculus has been shown to be
complete with respect to (conditional) causal effect identifiability,
$z$-identifiability, $g$-identifiability and transportability, it follows
that \dosearch{} is complete for these problems as well.
\begin{theorem}[Completeness] \label{thm:completeness2} If
\emph{\dosearch{}} returns NA in the settings in Rows~1--4 and 6 in
Table~\ref{tab:stateofart}, then the query is non-identifiable.
\end{theorem}
\begin{proof}
Do-calculus has been shown to be complete in these settings.  The rules of
probability calculus encode what is used in the algorithms as can be seen
for example from the proofs of Theorem~7 and Lemmas~4--8 of
\citet{Shpitser}.
\end{proof}
It is not known whether the rules implemented in \dosearch{} are sufficient
for other more general identifiability problems since it is conceivable that
some additional rules might exist that would be required to achieve
completeness.  One such generalization is the inclusion of missing data in
the causal model, which we present in Section~\ref{sect:missingness}. 
However, if one were to show that do-calculus (or any other set of rules
included in \dosearch{}) is complete for some special case of the
generalized identifiability problem, then \dosearch{} would be complete for
this problem as well.  In the following sections we will use the term
``identifiable by \dosearch{}'' to refer to causal queries that can be
identified by \dosearch{}.

\section{Extension to missing data problems} \label{sect:missingness}

The SCM framework can be extended to describe missing data mechanisms.  For
each variable $V_i$, two special vertices are added to the causal graph. 
The vertex $V_i^*$ is the observed proxy variable which is linked to the
true variable $V_i$ via the missingness mechanism \citep{missing,Mohan2013}:
\begin{equation} \label{eq:missingness}
V_i^* = \begin{cases}
  V_i, & \mathrm{if}\; R_{V_i} = 1, \\
  \NA,  & \mathrm{if}\; R_{V_i} = 0,
  \end{cases}
\end{equation}
where $\NA$ denotes a missing value and $R_{V_i}$ is called the response
indicator (of $V_i$).  In other words, the variable $V_i^*$ that is actually
observed matches the true value $V_i$ if it is not missing ($R_{V_i} = 1$). 
We note that in this formulation, each true variable has its own response
indicator, meaning that we do not consider shared indicators between
variables or multiple indicators for a single variable. 
Figure~\ref{fig:m2examples} in Section~\ref{sect:systematic} depicts some
examples of graphs containing missing data mechanisms.  Furthermore, if
there is no missingness associated with a given variable $V_i$ meaning that
it is fully observed, the corresponding response indicator $R_{V_i}$ always
has the value $1$.  The omission of a proxy variable and a response
indicators of a specific variable from a graph encodes the assumption that
the variable in question if fully observed.  Note that intervention nodes
are added for true variables and response indicators but not for proxy
variables.  On a symbolic level one could intervene on proxy variables,
however we are only interested in interventions that keep
Equation~\ref{eq:missingness} intact.

The observed vertices of the causal diagram can be partitioned into three
categories
\[
\+ V = \+ V^\true \cup \+ V^* \cup \+ V^r,
\]
where $\+ V^\true$ is the set of true variables, $\+ V^*$ is the set of
proxy variables and $\+ V^r$ is the set of response indicators.  For any
subset $\+ Z \subset \+ V$ we define the same partition via $\+ Z^\true = \+
Z \cap \+ V^\true$, $\+ Z^* = \+ Z \cap \+ V^*$ and $\+ Z^r = \+ Z \cap \+
V^r$.

The definition of the response indicator connects a proxy variable and a
true variable.  Typically this connection is only of interest for those
variables $V_i$ that have missing data, meaning that $P(R_{V_i} = 1) < 1$. 
Furthermore, we often utilize a proxy variable corresponding to a specific
true variable and conversely, a true variable corresponding to a specific
proxy variable.  We define this correspondence explicitly in the following
way
\begin{align*}
  \+ Z^\parprox &= \{V_i^* \in \+ V^* \mid V_i \in \+ Z^\true, P(R_{V_i} = 1) < 1\}, \\
  \+ Z^\proxpar &= \{V_i \in \+ V^\true \mid V_i^* \in \+ Z^*, P(R_{V_i} = 1) < 1\}.
\end{align*}
Similarly, given a set $\+ Z$ we often require the set of the response
indicators that define the missingness mechanism for the true variables that
are member of $\+ Z$.  This set is defined as follows
\[
 \+ R_{\+Z} = \{R_{V_i} \in \+ V^r \mid V_i \in \+ Z^\true, P(R_{V_i} = 1) < 1 \}.
\]
It is important to note the difference between the sets $\+ Z^r$ and $\+
R_{\+Z}$; the first set denotes the set of response indicators that are
members of $\+ Z$ while the second gives the corresponding response
indicators for the true variables that are member of $\+ Z$.

Our method is also capable of processing queries when the causal graph
contains missing data mechanisms where the sets $\+ A_i, \+ B_i$ and $\+
C_i$ of some of the input distributions may be restricted to contain
observed variables in $\+ V^* \cup \+ V^r$.  An active response indicator
$R_{V_i} = 1$ is denoted by $R_{V_i}^1$.  Similarly, for sets of response
indicators $\+ R_{\+Z}^1$ denotes that all indicators in the set are active. 
Proxy variables are not explicitly shown in graphs for clarity.

Determining identifiability is challenging under missing data.  As evidence
of this, even some non-interventional queries require the application of
do-calculus \citep{mohan2018}.  Furthermore, the rules used in the search of
Table~\ref{tab:rules} are no longer sufficient and deriving the desired
quantity necessitates the use of additional rules that stem from the
definition of the proxy variables and the response indicator.  Each new true
variable also has a higher impact on computational complexity, since the
corresponding response indicator and proxy variable are always added to the
graph as well.
\begin{table}[t!]
\centering
\resizebox{1.00\textwidth}{!}{%
\begin{small}
\begin{tabular}{llll}
\hline
 Rule & Additional Input & Output & Description\\
\hline
  $1+$ && $P(\+ Y \mid \doo(\+ X), \+ Z, \+ W)$ & Insertion of observations \\
  $1-$ && $P(\+ Y \mid \doo(\+ X), \+ W \setminus \+ Z)$ & Deletion of observations \\
  $2+$ && $P(\+ Y \mid \doo(\+ X,\+ Z), \+ W \setminus \+ Z)$ & Obs. to action exchange \\
  $2-$ && $P(\+ Y \mid \doo(\+ X \setminus \+ Z),\+ Z, \+ W)$ & Action to obs. exchange \\
  $3+$ && $P(\+ Y \mid \doo(\+ X,\+ Z), \+ W)$ & Insertion of actions \\
  $3-$ && $P(\+ Y \mid \doo(\+ X \setminus \+ Z), \+ W)$ & Deletion of actions \\
\hline
  $ 4$ && $P(\+ Y \setminus \+ Z \mid \doo(\+ X),\+ W)$ & Marginalization \\
  $ 5$ && $P(\+ Y \setminus \+ Z \mid \doo(\+ X),\+ Z,\+ W)$ & Conditioning \\
  $6+$ & $P(\+ Z \mid \doo(\+X),\+ W \setminus \+ Z)$ & $P(\+Y, \+Z \mid \doo(\+ X), \+ W \setminus \+ Z)$ & Chain rule multiplication \\
  $6-$ & $P(\+ Z \mid \doo(\+ X), \+ Y,\+ W)$  & $P(\+Y, \+Z \mid \doo(\+ X),\+W)$ & Chain rule multiplication \\
\hline
  $7+$  & $P(\+Z \mid \doo(\+X),\+W)$                      & $P(\+ Y \setminus \+ Z \mid \doo(\+X),\+Z,\+W)$ & Chain rule conditioning (numerator)\\
  $7-$  & $P(\+Z \mid \doo(\+X),\+W, \+ Y \setminus \+ Z)$ & $P(\+ Y \setminus \+ Z \mid \doo(\+X), \+W)$ & Chain rule conditioning (numerator)\\
  $8+$  & $P(\+Y, \+ Z \mid \doo(\+X),\+W)$                & $P(\+ Z \mid \doo(\+X),\+W, \+ Y)$ & Chain rule conditioning (denominator) \\
  $8-$  & $P(\+Y, \+ Z \mid \doo(\+X),\+W \setminus \+ Z)$ & $P(\+ Z \mid \doo(\+X),\+W)$ & Chain rule conditioning (denominator) \\
  $9+$  && $P(\+ Y \mid \doo(\+X),\+W \setminus \+ R_{\+ Z}, \+ R_{\+ Z}^1)$ & Enable response indicators \\
  $9-$  && $P(\+ Y \setminus \+ R_{\+ Z}, \+ R_{\+ Z}^1 \mid \doo(\+ X),\+ W)$ & Enable response indicators  \\
  $10+$ && $P(\+ Y \mid \doo(\+X),\+ W \setminus \+ Z^*,\+ Z^\proxpar)$ & Proxy variable exchange \\
  $10-$ && $P(\+ Y \setminus \+ Z^*,\+Z^\proxpar \mid \doo(\+X),\+W)$ & Proxy variable exchange \\
\hline
\end{tabular}
\end{small}}
\caption{Extended set of rules for missing data problems used to manipulate
input distributions of the form $P(\+ Y \mid \doo(\+ X), \+ W)$.  Rules
$1\pm, 2\pm, 3\pm, 4, 5$ and $6\pm$ are the same as in
Table~\ref{tab:rules}.  For Rules~$7\pm$ and $8\pm$, the additional input
has also been identified.  The sets $\+ Y, \+ X$ and $\+ W$ are disjoint. 
Sets $\+ Y$ and $\+ W$ may contain true variables, proxy variables and
response indicators.  Set $\+ X$ may only contain true variables and
response indicators.  The roles of the sets $\+ Z$ and $\+ R_{\+ Z}$ depend
on the rule being applied (see Table~\ref{tab:md_validsets}).}
\label{tab:md_rules}
\end{table}
\noindent
Table~\ref{tab:md_rules} extends the set of rules of Table~\ref{tab:rules}
to missing data problems by providing manipulations related to the
missingness mechanism.  Rules $7\pm$ and $8\pm$ perform conditioning using
the chain rule.  These rules are necessary in the case that set $\+ Y$
contains missing data mechanisms that have been enabled and thus cannot be
marginalized over by using Rule~$5$.

Rules $9\pm$ are used to enable response indicators, which then facilitates
the use of Rules~$10\pm$.  These last two rules exchange proxy variables to
their true counterparts when the corresponding response indicators are
enabled.  For example, under the conditions specified in
Table~\ref{tab:md_validsets}, Rule~$9+$ can be applied on $P(Y,X^*\mid
R_X)$ to first obtain $P(Y,X^* \mid R^1_X)$ by enabling $R_X$.  Then,
Rule~$10+$ can applied to this distribution to obtain $P(Y,X \mid R^1_X)$ by
exchanging $X^*$ for $X$.

Similarly to Table~\ref{tab:validsets}, Table~\ref{tab:md_validsets}
outlines the valid subsets $\+ Z$ for applying the extended rules of
Table~\ref{tab:md_rules}.  A major difference to the original validity and
termination conditions is the addition of the missing data condition that
outlines the additional requirements that must be satisfied when missingness
mechanisms are present.  For the rules that are shared by
Tables~\ref{tab:rules} and~\ref{tab:md_rules}, the missing data condition
ensures that a true variable and its proxy counterpart never appear in the
same term at the same time.  For example, we cannot add an intervention on
$X$ to $P(X^*)$.  It also ensures that we do not carry out summation over
enabled response indicators in the case of rules 4 and 5.  When applying
Rules~$9\pm$, the condition also ensures that we do not attempt to enable a
response indicator that is already enabled.  For Rules~$10\pm$, the
conditions guarantee that a proxy can only be exchanged to its true
counterpart if its corresponding response indicator is enabled and present
in the input term.
\begin{table}[t!]
\centering
\begin{small}
\begin{tabular}{l@{\hspace{6pt}}l@{\hspace{6pt}}ll}
\hline
Rule & Validity cond. & Missing data condition & Term. cond. \\
\hline
$1+$ &   $\+ Z \cap \+ T = \emptyset$                              & $\+ Z
\cap \left(\+ T^\proxpar \cup \+ T^\parprox \cup \+ Z^\proxpar \cup \+
Z^\parprox\right) = \emptyset$ & \\
$1-$ &   $\+ Z \subseteq \+ W$
&         & $\+ W = \emptyset$ \\
$2+$ &   $\+ Z \subseteq \+ W$                               & $\+ Z \cap \+
W^* = \emptyset$      & $\+ W = \emptyset$ \\
$2-$ &   $\+ Z \subseteq \+ X$
&         & $\+ X = \emptyset$ \\
$3+$ &   $\+ Z \cap \+ T = \emptyset$   & $\+ Z \cap \left(\+ T^\proxpar
\cup \+ T^\parprox \cup \+ Z^\proxpar \cup \+ Z^\parprox\right) = \emptyset$
& \\
$3-$ &   $\+ Z \subseteq \+ X$
&         & $\+ X = \emptyset$ \\ \hline
$4$  &   $\+ Z \subset \+ Y$
& $\+ Z \cap (\+ R^a \cap \+ Y) = \emptyset$        & $|\+ Y| = 1$ \\
$5$  &   $\+ Z \subset \+ Y$
& $ (\+ Y \setminus \+ Z) \cap (\+ R^a \cap \+ Y) = \emptyset$ & $|\+ Y| =
1$ \\
$6+$ &   $\+ Z \subseteq \+ W$                               &         & $\+
W = \emptyset$ \\
$6-$ &   $\+ Z \cap \+ T= \emptyset$  & $\+ Z \cap \left(\+ T^\proxpar \cup
\+ T^\parprox \cup \+ Z^\proxpar \cup \+ Z^\parprox\right) = \emptyset$ & \\
\hline
$7+$ &&  $\+ Z \subset \+ Y$                                          & $|\+
Y| = 1$ \\
$7-$ &&  $\+ Z \subset \+ Y$                                          & $|\+
Y| = 1$ \\
$8+$ &&  $\+ Z \cap \+ T = \emptyset$                                 & \\
$8-$ &&  $\+ Z \subseteq \+ W$                                        & $\+
W = \emptyset$\\
$9+$ &&  $\+ R_{\+ Z} \subseteq \+ W^r, \+ R_{\+ Z} \cap \+ R^a = \emptyset$
& $\+ W^r = \emptyset$ \\
$9-$ &&  $\+ R_{\+ Z} \subseteq \+ Y^r,\; \+ R_{\+ Z} \cap \+ R^a =
\emptyset$                            & $\+ Y^r = \emptyset$ \\
$10+$ && $\+ Z^* \subseteq \+ W^*, \+ R_{\+ Z^\proxpar} \subseteq \+ R^a, \+
R_{\+ Z^\proxpar} \subseteq \+ W^r$                 & $\+ R^a = \emptyset$
\\
$10-$ && $\+ Z^* \subseteq \+ Y^*,\; \+ R_{\+ Z^\proxpar} \subseteq \+ R^a,
(\+ R_{\+ Z^\proxpar} \subseteq \+ W^r \textrm{ or } \+ R_{\+ Z^\proxpar}
\subseteq\+ Y^r) $                 & $\+ R^a = \emptyset$ \\
\hline
\end{tabular}
\end{small}

\caption{The conditions for the enumerated subset $\+ Z$ for applying
  the rules of Table~\ref{tab:md_rules} to a term in the input column.  Here
  $\+ T = \+ Y \cup \+ X \cup \+ W$ and the sets $\+ Y, \+ X$ and $\+ W$ are
  those present in the input term $P(\+ Y \mid \doo(\+ X), \+ W)$.  Active
  response indicators of the input are denoted by $\+ R^a$.  For Rules~$6\pm,
  7\pm$ and $8\pm$, the conditions specify valid variables of the second
  required term.  Validity conditions for Rules~$1\pm,2\pm,3\pm,4,5$ and
  $6\pm$ are the same as in Table~\ref{tab:validsets}.}

\label{tab:md_validsets}
\end{table}

Additional termination conditions also apply to the new rules and their
correctness is easily verified.
\begin{lemma} \label{lem:md_termination} Let $G$ be a semi-Markovian graph
and let $\+ Y, \+ X$ and $\+ W$ be disjoint subsets of $\+ V$.  Then all of
the following are true:
\begin{itemize}
\item[(i)]{If $\+ |\+ Y| = 1$, then Rules~$7\pm$ of Table~\ref{tab:md_rules} cannot be used. }
\item[(ii)]{If $\+ W = \emptyset$, then Rule~$8-$ of Table~\ref{tab:md_rules} cannot be used. }
\item[(iii)]{If $\+ W^r = \emptyset$ then Rule~$9+$ of Table~\ref{tab:md_rules} cannot be used.}
\item[(iv)]{If $\+ Y^r = \emptyset$ then Rule~$9-$ of Table~\ref{tab:md_rules} cannot be used.}
\item[(v)]{If $\+ R^a = \emptyset$, then Rules~$10\pm$ of Table~\ref{tab:md_rules} cannot be used. }
\end{itemize}
\end{lemma}
\begin{proof}
For (i), the set $\+ Y$ only has a single vertex, so it cannot have a
non-empty subset.  For (ii), the set $\+ W$ is empty so no subset $\+ Z
\subset \+ W$ can exist for the second input.  For (iii), and the set $\+
W^r$ is empty so no assignment to value $1$ can be performed.  Similarly for
(iv), the set $\+ Y^r$ is empty so no assignment to value $1$ can be
performed.  For (v) the set of active response indicators $\+ R^a$ is empty,
so no transformation from proxy variables to true variables via the
missingness mechanism in Equation~\ref{eq:missingness} can take place.
\end{proof}
The task of selecting a suitable heuristic becomes more difficult when
missing data are involved with the identifiability problem.  The heuristic
approach of Section~\ref{sect:improvements} is no longer directly applicable
due to the relation between proxy variables, response indicators and true
variables.  The proximity function considers $X$ and $X^*$ as entirely
different variables despite their connection and does not prefer the
inclusion of response indicators.  If the heuristic is applied as such, the
search path will often involve a large number of manipulations which in turn
leads to complicated expressions.  For these reasons we do not apply a
heuristic to missing data problems, but expand terms in the order in which
they were identified.  The improvements described in
Section~\ref{sect:improvements} still apply.

It is straightforward to adapt \dosearch{} to the new extended set of rules. 
In the pseudo-code shown in Algorithm~\ref{prog:do-search}, we simply
replace all references to Tables~\ref{tab:rules} and \ref{tab:validsets} by
references to Tables~\ref{tab:md_rules} and Tables~\ref{tab:md_validsets},
respectively.  When the validity condition is checked, we also verify that
the missing data condition holds.  Lemma~\ref{lem:md_termination} guarantees
the correctness of the new termination criteria. 
Theorem~\ref{thm:nonidsimple} is also valid when the sets $\+ A_i$ are
replaced by $\+ A_i \cup \+ A_i^\proxpar$, since it may be possible to
exchange some proxy variable to a true variable that is present in the set
$\+ Y$ of the target $P(\+ Y \mid \doo(\+ X), \+ W)$.

\section[The dosearch package]{The \pkg{dosearch} package} \label{sect:package}

We implemented \dosearch{} (Algorithm~\ref{prog:do-search}) in
\proglang{C++} and constructed an \proglang{R} interface using the
\pkg{Rcpp} package \citep{Rcpp}.  This interface is provided by the
\proglang{R}~package \pkg{dosearch}.  Calling the search from \proglang{R}
is straightforward via the primary function that carries the name of
package.
\begin{Code}
dosearch(data, query, graph, 
  transportability, selection_bias, missing_data,
  control)
\end{Code}
The required inputs of the function are \code{data}, \code{query} and
\code{graph}.  Parameter \code{data} is used to encode the set $\+ P$ of
known input distributions of Algorithm~\ref{prog:do-search} as a character
string, where each distribution is separated by a new line.  For example, if
we have access to a set of  distributions $\+ P = \{P(W), P(Y \mid X), P(Z
\mid \doo(X), W)\}$, we would write
\begin{CodeChunk}
\begin{CodeInput}
R> data <- "
+    P(W)
+    P(Y|X)
+    P(Z|do(X),W)
+  "
\end{CodeInput}
\end{CodeChunk}
The $\doo(\cdot)$-operator can either precede or succeed conditioning
variables, but it must appear only once in a given term, meaning that
expressions such as \code{P(Y|do(A),B,do(C))} are not allowed, but should
instead be given as \code{P(Y|B,do(A,C))} or \code{P(Y|do(A,C),B)}.  If
variable sets are desired, each member of the set has to be included
explicitly.

Parameter \code{query} is used to describe the target $Q$ of
Algorithm~\ref{prog:do-search} as a character string, similarly as the
\code{data}.  If we are interested in identifying $P(Y \mid \doo(X), W)$
we would write
\begin{CodeChunk}
\begin{CodeInput}
R> query <- "P(Y|do(X),W)"
\end{CodeInput}
\end{CodeChunk}
Instead of describing distributions via text, it is also possible to use the
following structure that encodes the role of each variable via a numeric
vector:
\begin{CodeChunk}
\begin{CodeInput}
R> query <- c(Y = 0, X = 1, W = 2)
\end{CodeInput}
\end{CodeChunk}
Given a distribution of the form $P(\+ A \mid \doo(\+B),\+ C)$ and a
variable $V$, a value 0 means that $V \in \+ A$, value 1 means that $V \in
\+ B$ and value 2 means that $V \in \+ C$.  This format can also be used to
input \code{data} as a list of numeric vectors:
\begin{CodeChunk}
\begin{CodeInput}
R> data <- list(
+    c(W = 0),
+    c(Y = 0, X = 2),
+    c(Z = 0, X = 1, W = 2)
+  )
\end{CodeInput}
\end{CodeChunk}
Finally, \code{graph} encodes the semi-Markovian graph $G$ of the causal
model as a character string with each edge on its own line.  A directed edge
from $X$ to $Y$ is given as \code{X -> Y} and a bidirected edge between $X$
and $Y$ is given as \code{X <-> Y}.  Intervention nodes should not be given
explicitly, since they are added automatically after calling
\code{dosearch}.  Furthermore, only vertices with incoming or outgoing edges
should be included in \code{graph}.  A variable with no connected edges can
still appear in the input distributions and is automatically added to the
graph as well.  As an example, we can encode the graph of
Figure~\ref{fig:simple_search}(\subref{fig:intro_backdoor}) with an added
bidirected edge between $X$ and $Y$ as follows
\begin{CodeChunk}
\begin{CodeInput}
R> graph <- "
+    X -> Y
+    Z -> X
+    Z -> Y
+    X <-> Y
+  "
\end{CodeInput}
\end{CodeChunk}
Alternatively, one may use \pkg{igraph} graphs \citep{igraph} in the syntax
of the \pkg{causaleffect} package or DAGs created using the \pkg{dagitty}
package.
\begin{CodeChunk}
\begin{CodeInput}
R> library("igraph")
R> graph <- graph.formula(X -+ Y, Z -+ X, Z -+ Y, X -+ Y, Y -+ X)
R> graph <- set.edge.attribute(graph, "description", 4:5, "U")
\end{CodeInput}
\end{CodeChunk}
\begin{CodeChunk}
\begin{CodeInput}
R> library("dagitty")
R> graph <- dagitty("dag{X -> Y; Z -> X; Z -> Y; X <-> Y}")
\end{CodeInput}
\end{CodeChunk}
The next two optional parameters of \code{dosearch}, \code{transportability}
and \code{selection_bias}, are used to denote those vertices of $G$ that
should be understood as either transportability nodes or selection bias
nodes, respectively.  Providing these parameters may increase search
performance in relevant problems.  Both of these parameters should be given
as character strings, where individual variables are separated by a comma,
for example \code{transportability = "S,T"}.  Parameter \code{missing_data},
as the name suggests, is used to define missingness
mechanisms~\ref{eq:missingness} as a character string, where individual mechanisms
are separated by a comma.  In order to describe that $R_X$ is the response
indicator of $X$ we would write \code{R_X : X}, which also implicitly
defines that \code{X*} is the proxy variable of \code{X}.  Proxy variables
do not need to be manually described in \code{graph} as they are
automatically constructed based on the \code{missing_data} argument.

The list \code{control} can be used to set various additional parameters
that are not directly related to the identifiability problem itself, but
more so to the output of the search and other auxiliary details, such as
benchmarking and obtaining derivations such as
Figure~\ref{fig:simple_search}(\subref{fig:derivation}).  One such control
parameter determines whether to use the search heuristic or not
(\code{heuristic = FALSE} by default). 
Documentation of the \pkg{dosearch} package contains detailed information on
the full list of control parameters.  

The return object of \code{dosearch} is a list with three components by
default.  The first component, \code{identifiability}, is a logical value
that takes the value \code{TRUE} when the target distribution described by
\code{query} is identifiable from the inputs of \code{data}.  The second
component, \code{formula}, is a character string describing the target
distribution in terms of the inputs in \LaTeX~syntax if the target is
identifiable.  Otherwise this component is just an empty character string. 
The third component \code{call} contains the arguments of the original
function call.

\subsection{Simulations} \label{sec:simu}

Here we report the results of a simulation study to assess the running time
performance of \dosearch{} and the impact of the search space reduction
techniques as well as the search heuristic outlined in
Section~\ref{sect:improvements}.

Our synthetic simulation scenario consisted of $1000$ semi-Markovian causal
graphs of $10$~vertices that were generated at random by first generating a
random topological order of the vertices followed by a random lower
triangular adjacency matrices for both directed and bidirected
edges. Graphs without a directed path from $X$ to $Y$ were discarded. 
We sampled sequentially input distributions of the form $P(\+ A \mid
\doo(\+ B), \+ C)$ at random by generating disjoint subsets such that $\+ A$
is always non-empty.  This was continued until the target quantity $P(Y
\mid \doo(X))$ was found to be identifiable by the search.  Then for each
graph, we recorded the search times for the specific set of inputs that
first resulted in the query to be identified and for the last set such that
the target was non-identifiable.  In other words, each graph generates two
simulation instances, one for an identifiable query and one for a
non-identifiable query.  This setting directly corresponds to the setting of
partially overlapping experimental data sets discussed in
Section~\ref{sect:new_scenarios} for which no other algorithmic solutions
exist.

To understand the impact of the search heuristic and the various
improvements, we compare four different search configurations: the basic
\dosearch{} without the search heuristic or improvements\footnote{In this
configuration, terms are expanded in the order they were identified; the
conditions in Table~\ref{tab:validsets} are not checked.}, one that only
uses the search heuristic, one that only uses the improvements of
Section~\ref{sect:improvements} and one that uses them both.
\begin{figure}[t!]
  \centering
  \begin{subfigure}[t]{0.325\textwidth}
    \centering
    \includegraphics[width=1.0\textwidth]{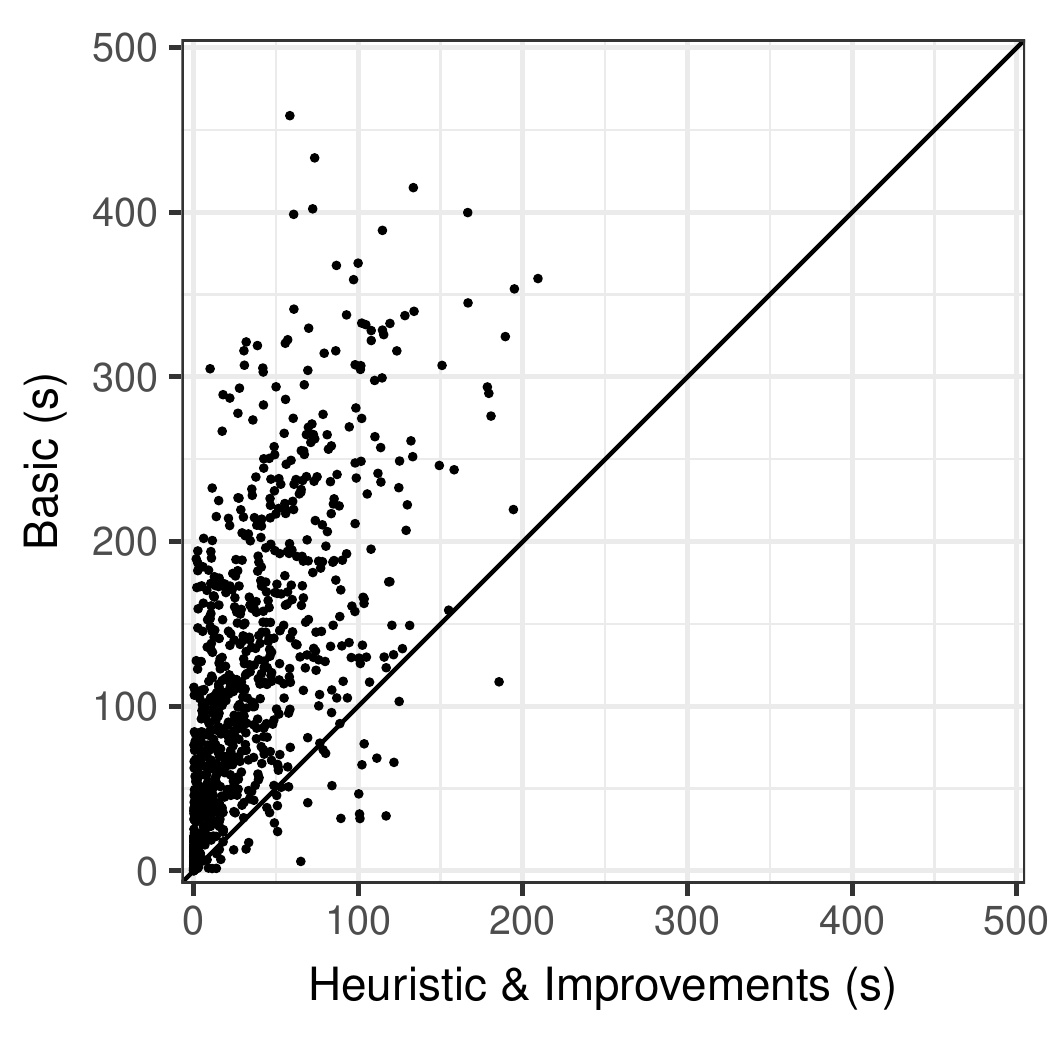}
    \caption{}
    \label{fig:scatter_id_hi}
  \end{subfigure}
  \hfill
  \begin{subfigure}[t]{0.325\textwidth}
    \centering
    \includegraphics[width=1.0\textwidth]{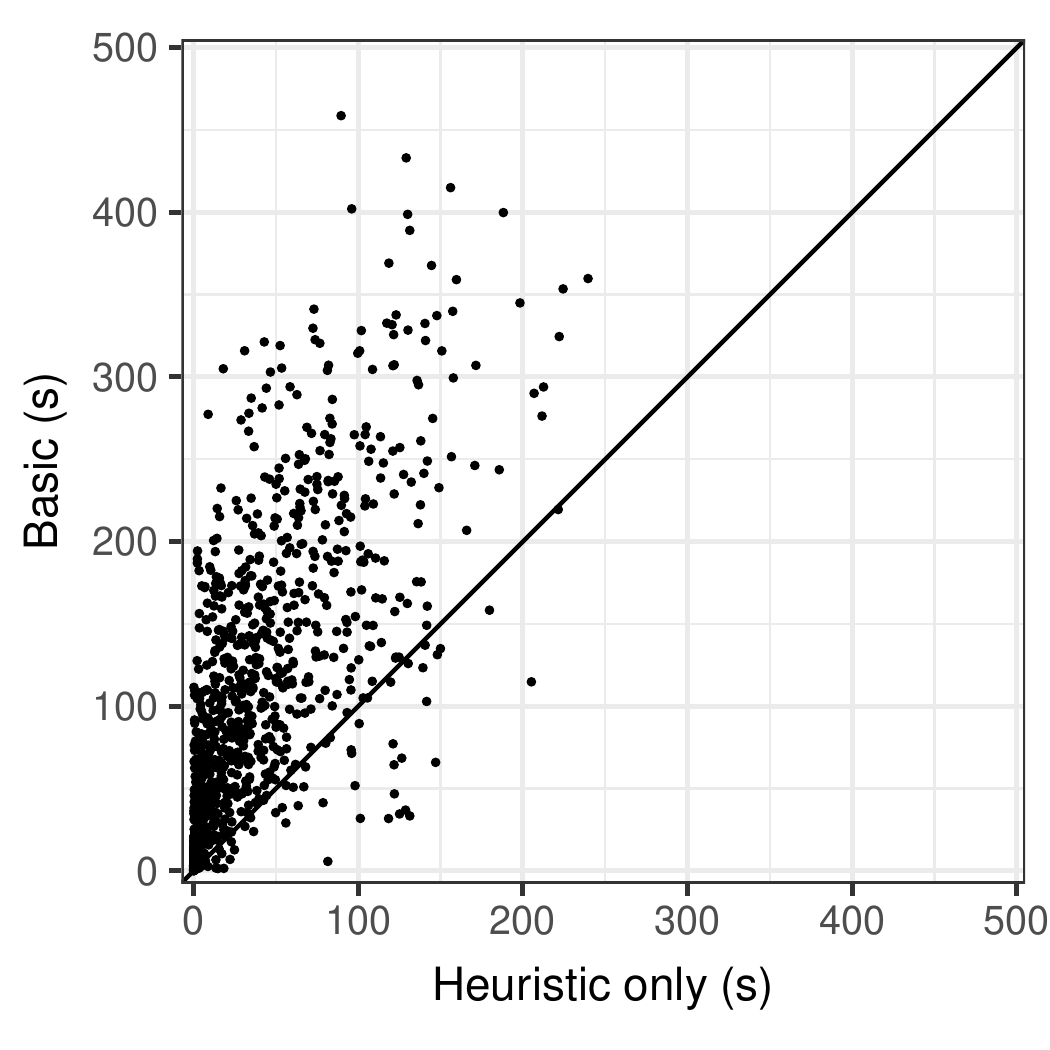}
    \caption{}
    \label{fig:scatter_id_h}
  \end{subfigure}
  \hfill
  \begin{subfigure}[t]{0.325\textwidth}
    \centering
    \includegraphics[width=1.0\textwidth]{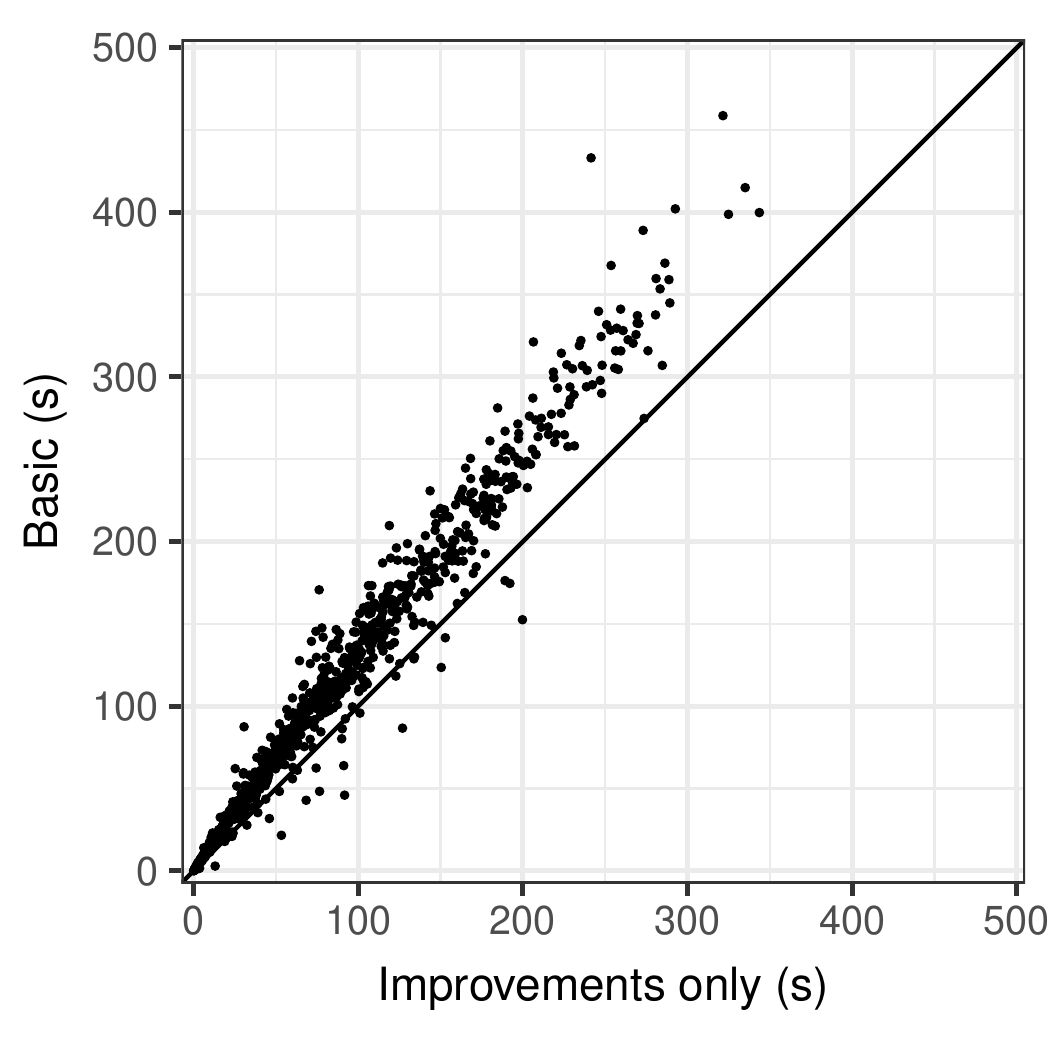}
    \caption{}
    \label{fig:scatter_id_i}
  \end{subfigure}
  \caption{Scatter plots of the search times from identifiable instances
    under different search configurations compared to the basic \dosearch{}
    without a heuristic or improvements.}
  \label{fig:scatter_id}
\end{figure}

Figure~\ref{fig:scatter_id} shows the search times of the configurations
compared to the basic configuration for identifiable instances.  Most
importantly, a vast majority of instances (96\%) are solved faster than the
basic configuration when both heuristics and improvements are used.  The
average search time with both heuristics and improvements enabled was 
$31.5$~seconds and $80.2$~seconds for the basic configuration.  The search 
heuristic provides the greatest benefit for these instances as can be seen from
Figure~\ref{fig:scatter_id}(\subref{fig:scatter_id_h}).  Using a heuristic
can sometimes hinder performance by leading the search astray and by causing
additional computational steps through the evaluation of the proximity
function.  For example, there is a small number of instances where the search
is over ten times slower than the basic configuration when using a heuristic.
Fortunately, there are several instances in the opposite direction, where the
heuristic provides over one hundred fold reduction in search time. Curiously, 
even using the improvements sometimes results in slower search times. 
This is most likely due to the elimination of Rule~1 of do-calculus, since it 
may be the case that the basic search is able to use this rule to reach the 
target distribution faster.  More importantly, 
Figure~\ref{fig:scatter_id}(\subref{fig:scatter_id_i}) shows that the
improvements clearly benefit the search.  Furthermore, the benefit tends to
increase as the instances get harder.

Figure~\ref{fig:scatter_nonid} shows the search times of the configurations
for non-identifiable instances.  Relying only on a search heuristic provides
no benefit here, as expected.  The improvements to the search are most
valuable for these instances, and in this scenario every non-identifiable
instance was solved faster than baseline using the improvements, and when
applied with the heuristic only one non-identifiable instance was slower 
than baseline. The average search time with both heuristic and improvements 
enabled was $134.6$~seconds and $182.5$~seconds for the basic configuration. 
The almost zero second instances are a result of Theorem~\ref{thm:nonidsimple} 
when no search has to be performed in order to determine the instance to be
non-identifiable.  The benefit of the improvements tends to increase as the
instances get harder also for these instances.
\begin{figure}[t!]
  \centering
  \begin{subfigure}[t]{0.325\textwidth}
    \centering
    \includegraphics[width=1.0\textwidth]{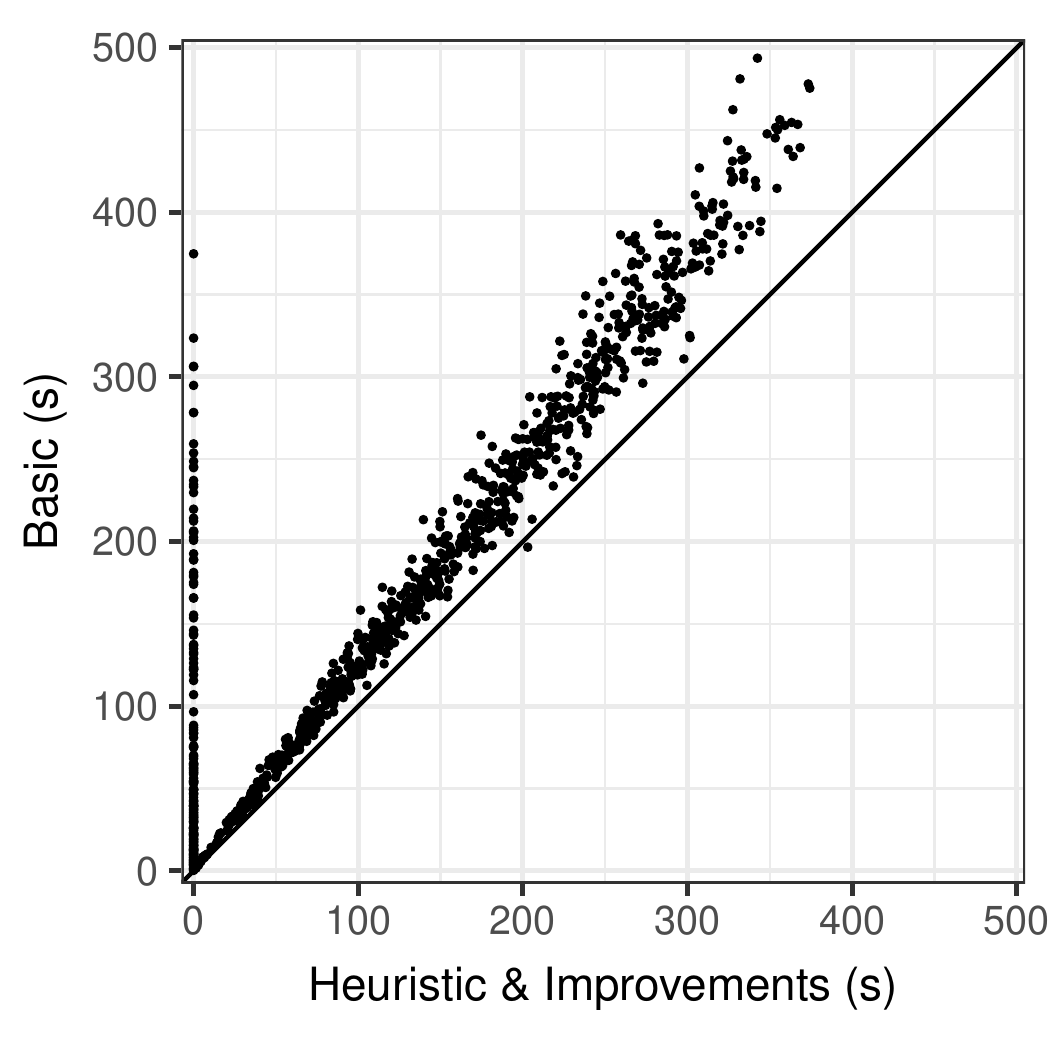}
    \caption{}
    \label{fig:scatter_nonid_hi}
  \end{subfigure}
  \hfill
  \begin{subfigure}[t]{0.325\textwidth}
    \centering
    \includegraphics[width=1.0\textwidth]{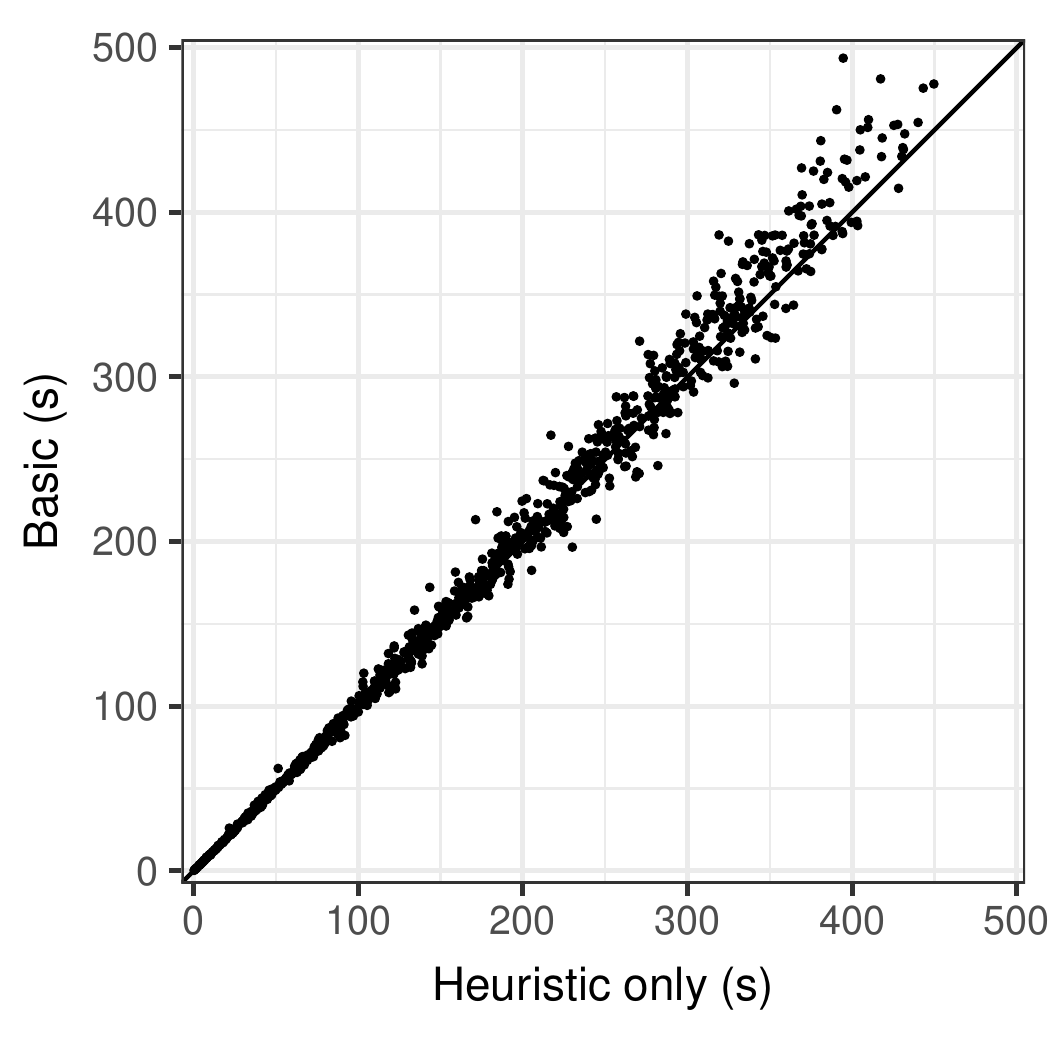}
    \caption{}
    \label{fig:scatter_nonid_h}
  \end{subfigure}
  \hfill
  \begin{subfigure}[t]{0.325\textwidth}
    \centering
    \includegraphics[width=1.0\textwidth]{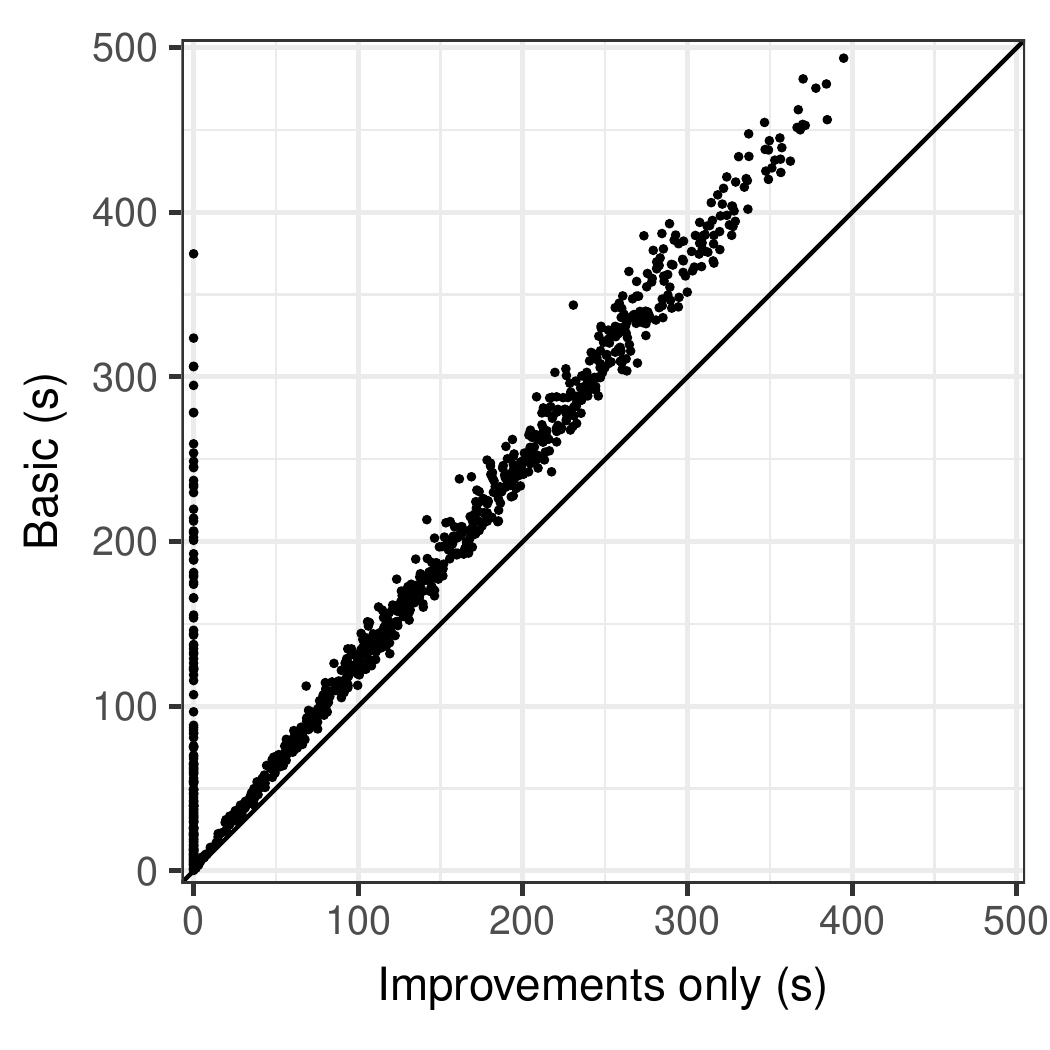}
    \caption{}
    \label{fig:scatter_nonid_i}
  \end{subfigure}
  \caption{Scatter plots of the search times from non-identifiable instances
    under different search configurations compared to the baseline
    configuration.  }
  \label{fig:scatter_nonid}
\end{figure}
\begin{figure}[t!]
  \centering
  \includegraphics[width=0.46\textwidth]{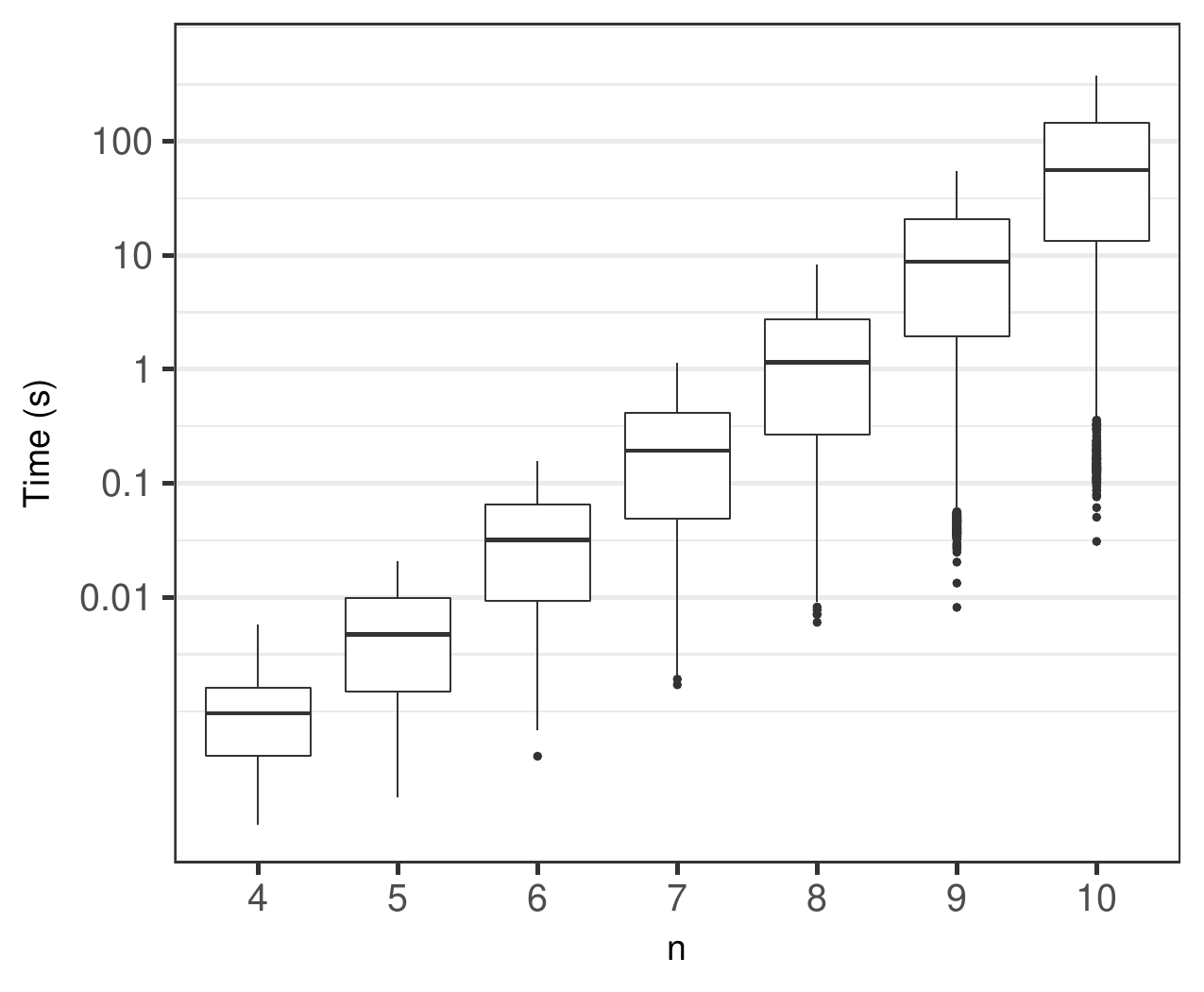}
  \caption{Boxplots of search times for both identifiable and
    non-identifiable instances in graphs of $n = 4,\ldots,10$~vertices.  The
    vertical axis uses a logarithmic scaling.  Instances where the search
    time was less than $10^{-4}$~seconds were omitted for clarity.}
  \label{fig:time_by_n}
\end{figure}

Finally we examined the average run time performance of \dosearch{}, with
all improvements and heuristics enabled.  We replicated the previously
described simulation scenario with the same number of instances ($1000$) for
graphs with up to $10$~vertices.  Figure~\ref{fig:time_by_n} shows the boxplots of
search times on a log-scale for graphs of different size, including both
identifiable and non-identifiable instances.  Note that for every graph size
there are a number of easily solvable instances that show up as outliers
in this plot. Instances with graphs of $10$~vertices  are solved routinely
in under $100$~seconds.  In this plot, the running times 
increase exponentially with increasing graph size (or number of variables).

\section{New causal effect identification results} \label{sect:new_results}

We present a number of results for various identifiability problems to
showcase the versatility of \dosearch{} with the accompanying \proglang{R}
code for some specific examples.

\subsection{Multiple data sources with partially overlapping variable sets}

Earlier generalizations of the identifiability problem assume nested
experiments or entire distributions with the exception of surrogate outcome
identifiability \citep{Tikka:surrogate} which also has its own intricate set
of assumptions regarding the available distributions.  None of these
assumptions are needed in \dosearch{} and it can be used to solve
identifiability problems from completely arbitrary collections of input
distributions.

We showcase identifiability from multiple data sources by three examples. 
The first example is the human resource management problem presented in the
introduction and shown in Figure~\ref{fig:hr}.  The question of interest was
the identifiability of $P(Y \mid \doo(X))$ from the data sources
$P(Y,B,E,X)$ and $P(A,B,X)$.  The answer can be provided with the following
lines of \proglang{R} code:
\begin{CodeChunk}
\begin{CodeInput}
R> library("dosearch")
R> data <- "
+    p(y,b,e,x)
+    p(a,b,x)
+  "
R> query <- "p(y|do(x))"
R> graph <-  "
+    e -> x
+    e -> y
+    a -> b
+    a -> x
+    x -> b
+    x -> y
+    b -> y
+  "
R> dosearch(data, query, graph, control = list(heuristic = TRUE))
\end{CodeInput}
\begin{CodeOutput}
\sum_{b,a}\left(p(a)\left(p(b|x,a)\sum_{e}
  \left(p(e)p(y|x,b,e)\right)\right)\right)
\end{CodeOutput}
\end{CodeChunk}
The result means that the causal effect is identifiable and the returned
formula is
\[
  \sum_{B,A}P(A)P(B \mid X,A)\sum_{E} P(E)P(Y|X,B,E).
\]
By running an additional line of code
\begin{CodeChunk}
\begin{CodeInput}
R> dosearch(data, "p(y,b,e,x,a)", graph, control = list(heuristic = TRUE))
\end{CodeInput}
\begin{CodeOutput}
The query p(y,b,e,x,a) is non-identifiable. 
\end{CodeOutput}
\end{CodeChunk}
we learn that the joint distribution $P(Y,B,E,X,A)$ is not identifiable. 
This rules out the possibility to solve the problem with \pkg{causaleffect}
which requires the joint distribution as an input.

In the second example we consider identifiability of $P(Y_1,Y_2 \mid
\doo(X_1,X_2))$ in the graph of
Figure~\ref{fig:obsandexp_examples}(\subref{fig:obsandexp_1}) from $P(\+
V)$, $P(Y_1,Y_2 \mid \doo(X_1),Z,W,X_2)$, $P(W \mid \doo(X_1, X_2))$ and
$P(Z \mid \doo(X_2))$.  The target quantity is identifiable and
\dosearch{} produces the following formula for it
\[
\sum_{Z,W} P(Y_1,Y_2 \mid \doo(X_1), Z, W, X_2)P(Z \mid \doo(X_2))P(W
\mid \doo(X_2, X_1))
\]
In the third example we consider identifiability of $P(Y_1,Y_2 \mid
\doo(X_1,X_2))$ in the graph of
Figure~\ref{fig:obsandexp_examples}(\subref{fig:obsandexp_2}) from $P(\+
V)$, $P(Y_1 \mid \doo(X_1),Y_2,W,Z,X_2)$, $P(X_2,W \mid \doo(X_1))$,
$P(X_2 \mid \doo(X_1,W))$, $P(Y_2 \mid \doo(X_1),Z,W,X_2)$, $P(Y_2
\mid \doo(Z),X_1,W,X_2)$, and $P(Y_1,Y_2 \mid \doo(Z),W,X_1,X_2)$. 
Again, the target quantity is identifiable and \dosearch{} outputs the
following formula
\begin{align*}
\sum_{W}&\left(P(W \mid \doo(X_1),X_2) \sum_{X_2} P(X_2 \mid
\doo(X_1,W)) \times \right.  \\
&\left.\frac{\sum_{Z}\!  P(X_2,W,Z \mid X_1)P(Y_1,Y_2 \mid
\doo(X_1),X_2,W,Z)}{\sum_{Y^\prime_1,Y^\prime_2,Z}\!  P(X_2,W,Z \mid
X_1)P(Y^\prime_1,Y^\prime_2 \mid \doo(X_1),X_2,W,Z)} \right).
\end{align*}
This example shows that a heuristic approach can also help us to find
shorter formulas.  If we run \dosearch{} again without the heuristic in this
instance, the output formula is instead
\begin{align*}
&\sum_{W,Z} \left(P(Z)P(W \mid X_2,X_1,Z) \sum_{X_2} P(X_2 \mid X_1,Z)
\sum_{Y_2} P(Y_2 \mid \doo(X_1),X_2,W,Z)\right.  \, \times \\
&\quad\left.  P(Y_1 \mid \doo(X_1),X_2,Y_2,W,Z) \frac{P(Y_2 \mid
\doo(X_1),X_2,W,Z)P(Y_1 \mid \doo(X_1),X_2,Y_2,W,Z)}{\sum_{Y^\prime_2}
P(Y^\prime_2 \mid \doo(X_1),X_2,W,Z)P(Y_1 \mid
\doo(X_1),X_2,Y^\prime_2,W,Z)} \right).
\end{align*}
We can run these examples in \proglang{R} by writing
\begin{CodeChunk}
\begin{CodeInput}
R> data <- "
+    p(x_1,y_1,x_2,y_2,z,w)
+    p(y_1,y_2|z,w,x_2,do(x_1))
+    p(y_2|y_1,z,w,x_2,do(x_1))
+    p(w|do(x_1,x_2))
+    p(z|do(x_2))
+  "
R> query <- "p(y_1,y_2|do(x_1,x_2))"
R> graph <- "
+    z -> y_1
+    w -> y_1
+    y_1 -> y_2
+    x_2 -> z
+    x_1 -> w
+    y_1 <-> x_1
+    y_1 <-> y_2
+    y_2 <-> z
+    y_1 <-> w
+    y_2 <-> w
+  "
R> dosearch(data, query, graph, control = list(heuristic = TRUE))
\end{CodeInput}
\begin{CodeOutput}
\sum_{z,w}\left(p(y_1,y_2|do(x_1),z,w,x_2)
  \left(p(z|do(x_2))p(w|do(x_2,x_1))\right)\right)
\end{CodeOutput}
\end{CodeChunk}
and
\begin{CodeChunk}
\begin{CodeInput}
R> data <- "
+    p(x_1,y_1,x_2,y_2,z,w)
+    p(y_1,y_2|w,x_1,x_2,do(z))
+    p(y_1|y_2,w,z,x_2,do(x_1))
+    p(y_2|x_1,w,x_2,do(z))
+    p(x_2,w|do(x_1))
+    p(x_2|do(x_1,w))
+    p(y_2|z,w,x_2,do(x_1))
+  "
R> query <- "p(y_1,y_2|do(x_1,x_2))"
R> graph <- "
+    y_2 -> y_1
+    w -> y_1
+    x_1 -> x_2
+    x_1 -> y_2
+    z -> y_2
+    w -> y_2
+    x_2 -> w
+    x_1 <-> y_1
+    x_1 <-> y_2
+    y_1 <-> z
+    x_2 <-> z
+  "
R> dosearch(data, query, graph, control = list(heuristic = TRUE))
\end{CodeInput}
\begin{CodeOutput}
\sum_{w}\left(p(w|do(x_1),x_2)\sum_{x_2}\left(p(x_2|do(w,x_1))
  \frac{\sum_{z}\left(p(x_2,w,z|x_1)p(y_1,y_2|do(x_1),x_2,w,z)\right)}
  {\sum_{y_1,y_2} \sum_{z}\left(p(x_2,w,z|x_1)p(y_1,y_2|do(x_1),x_2,w,z)
  \right)}\right)\right)
\end{CodeOutput}
\end{CodeChunk}
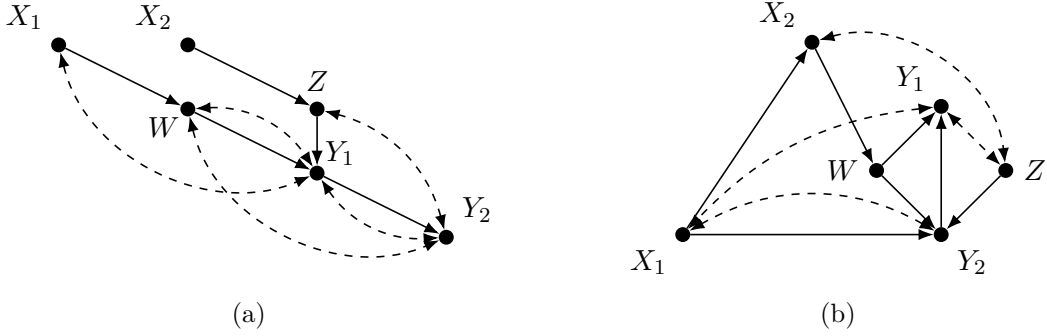
\begin{figure}[t!]
  \begin{subfigure}[t]{0.49\textwidth}
    \centering
    \begin{tikzpicture}[scale=1.7]
    \node [dot = {0}{0}{X_1}{above left}] at (0,0) {};
    \node [dot = {0}{0}{X_2}{above left}] at (1,0) {};
    \node [dot = {-0.1}{-0.1}{Y_1}{above right}] at (2,-1) {};
    \node [dot = {0}{0}{Y_2}{above right}] at (3,-1.5) {};
    \node [dot = {0}{0}{Z}{above}] at (2,-0.5) {};
    \node [dot = {0.1}{0.1}{W}{below left}] at (1,-0.5) {};

    \draw [->] (X_1) -- (W);
    \draw [->] (X_2) -- (Z);
    \draw [->] (W) -- (Y_1);
    \draw [->] (Z) -- (Y_1);
    \draw [->] (Y_1) -- (Y_2);

    \draw [<->,dashed] (X_1) to [bend right=50] (Y_1);
    \draw [<->,dashed] (W) to [bend right=50] (Y_2);
    \draw [<->,dashed] (W) to [bend left=30] (Y_1);
    \draw [<->,dashed] (Y_1) to [bend right=30] (Y_2);
    \draw [<->,dashed] (Y_2) to [bend right=30] (Z);

    \end{tikzpicture}
    \caption{}
    \label{fig:obsandexp_1}
  \end{subfigure}
  \begin{subfigure}[t]{0.49\textwidth}
    \centering
    \begin{tikzpicture}[scale=1.7]
    \node [dot = {0}{0}{X_1}{below left}] at (0,0) {};
    \node [dot = {0}{0}{X_2}{above left}] at (1,1.5) {};
    \node [dot = {0}{0}{Y_1}{above left}] at (2,1) {};
    \node [dot = {0}{0}{Y_2}{below right}] at (2,0) {};
    \node [dot = {0}{0}{Z}{right}] at (2.5,0.5) {};
    \node [dot = {0}{0}{W}{left}] at (1.5,0.5) {};

    \draw [->] (X_2) -- (W);
    \draw [->] (X_1) -- (Y_2);
    \draw [->] (X_1) -- (X_2);
    \draw [->] (W) -- (Y_1);
    \draw [->] (W) -- (Y_2);
    \draw [->] (Y_2) -- (Y_1);
    \draw [->] (Z) -- (Y_2);

    \draw [<->,dashed] (Y_2) to [bend right=30] (X_1);
    \draw [<->,dashed] (X_2) to [bend left=50] (Z);
    \draw [<->,dashed] (Y_1) to (Z);
    \draw [<->,dashed] (Y_1) to [bend right=20] (X_1);
    \end{tikzpicture}
    \caption{}
    \label{fig:obsandexp_2}
  \end{subfigure}
  \caption{Graphs for the examples on identifiability problems combining
    both observational and experimental distributions.}
  \label{fig:obsandexp_examples}
\end{figure}
\subsection{Combining transportability and selection bias}

Input distributions that originate from multiple sources while being
simultaneously affected by selection bias can be considered with
\dosearch{}.  This kind of problem cannot be solved with algorithms RC or
TR$^{\textrm{mz}}$ of Table~\ref{tab:stateofart}.  As an example we consider
one source domain and a target domain with two input data sets: a biased
distribution $P(X,Y,Z \mid S)$ from the target domain and an unbiased
experimental distribution $P(Y,Z \mid \doo(X), T)$ from the source domain. 
We evaluate the query $P(Y \mid \doo(X))$ in the graph of
Figure~\ref{fig:combine} using these inputs.  In the figure transportability
node $T$ is depicted as a gray square and selection bias node $S$ is
depicted as an open double circle.
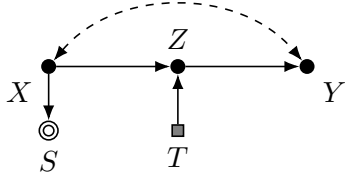
\begin{figure}[t!]
  \centering
  \begin{tikzpicture}[scale=1.7]
    \node [dot = {0}{0}{X}{below left}] at (0,0) {};
    \node [dot = {0}{0}{Z}{above}] at (1,0) {};
    \node [dot = {0}{0}{Y}{below right}] at (2,0) {};
    \node [sb = {0}{0}{S}{below}] at (0,-0.5) {};
    \draw [black, inner sep = 0.5pt] (0,-0.5) circle (1.15pt);
    \node [tr = {0}{0}{T}{below}] at (1,-0.5) {};
    \draw [->] (X) -- (Z);
    \draw [->] (Z) -- (Y);
    \draw [->] (X) -- (S);
    \draw [->] (T) -- (Z);
    \draw [<->,dashed] (X) to [bend left=50]  (Y);
  \end{tikzpicture}
  \caption{Graph that contains both selection bias and transportability nodes.}
  \label{fig:combine}
\end{figure}
\noindent
The query is identifiable and \dosearch{} outputs the following formula for
it
\[
 P(Y \mid \doo(X)) = \sum_{Z} P(Y \mid \doo(X),Z,T) \sum_{Y^\prime}
 P(Z,Y^\prime \mid X,S).
\]
In \proglang{R} we may write
\begin{CodeChunk}
\begin{CodeInput}
R> data <- "
+    p(x,z,y|s)
+    p(y,z|t,do(x))
+  "
R> query <- "p(y|do(x))"
R> graph <- "
+    x -> z
+    z -> y
+    x -> s
+    t -> z
+    x <-> y
+  "
R> dosearch(data, query, graph, 
+    transportability = "t", selection_bias = "s"
+    control = list(heuristic = TRUE))
\end{CodeInput}
\begin{CodeOutput}
\sum_{z}\left(p(y|do(x),z,t)\sum_{y}p(z,y|x,s)\right)
\end{CodeOutput}
\end{CodeChunk}

\subsection{Recovering from multiple sources of selection bias}

We present an example where bias originates from two sources with two input
data sets: a distribution affected by both biasing mechanisms
$P(X,Y,Z,W_1,W_2 \mid S_1,S_2)$ and a distribution affected only by a
single bias source $P(Z \mid S_1)$.  We evaluate the query $P(Y \mid
\doo(X))$ in the graph of Figure~\ref{fig:twobias} using the inputs.
\begin{figure}[t!]
  \centering
  \begin{tikzpicture}[scale=1.35]
    \node [dot = {0}{0}{X}{below}] at (0,0) {};
    \node [dot = {0}{0}{Z}{above}] at (2,2) {};
    \node [dot = {0}{0}{Y}{below right}] at (3,0) {};
    \node [dot = {0}{0}{W_2}{below right}] at (1,1) {};
    \node [dot = {0}{0}{W_1}{above left}] at (-0.5,1.5) {};
    \node [sb = {0}{0}{S_1}{below}] at (-1,0) {};
    \draw [black, inner sep = 0.5pt] (-1,-0) circle (1.40pt);
    \node [sb = {0}{0}{S_2}{below}] at (2,1) {};
    \draw [black, inner sep = 0.5pt] (2,1) circle (1.40pt);
    \draw [->] (X) -- (Y);
    \draw [->] (W_2) -- (X);
    \draw [->] (Z) -- (W_2);
    \draw [->] (Z) -- (Y);
    \draw [->] (Z) -- (S_2);
    \draw [->] (W_1) -- (X);
    \draw [->] (W_1) -- (W_2);
    \draw [->] (W_1) -- (S_1);
  \end{tikzpicture}
  \caption{Graph where two selection bias nodes are present.}
  \label{fig:twobias}
\end{figure}
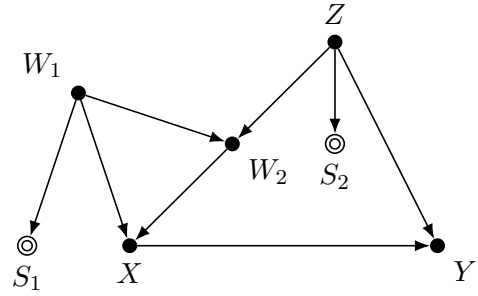
\noindent
The query is identifiable and the following formula is obtained by
\dosearch{}
\[
  \sum_{Z} P(Z \mid S_1)P(Y \mid X,Z,W_1,W_2,S_1,S_2).
\]
The result can be obtained in \proglang{R} as follows.  In this case,
slightly nicer expression is produced when not using the heuristic
than when using the heuristic.
\begin{CodeChunk}
\begin{CodeInput}
R> data <- "
+    p(x,y,z,w_1,w_2|s_1,s_2)
+    p(z|s_1)
+  "
R> query <- "p(y|do(x))"
R> graph <- "
+    w_1 -> w_2
+    z -> w_2
+    x -> y
+    z -> y
+    z -> s_2
+    w_1 -> x
+    w_2 -> x
+    w_1 -> s_1
+  "
R> dosearch(data, query, graph, selection_bias = "s_1, s_2")
\end{CodeInput}
\begin{CodeOutput}
\sum_{z}\left(p(z|s_1)p(y|w_2,x,w_1,z,s_1,s_2)\right)
\end{CodeOutput}
\end{CodeChunk}

\subsection{Systematic analysis of bivariate missing data problems}
\label{sect:systematic}

We apply \dosearch{} using the extended rule set of Table~\ref{tab:md_rules}
for all identifiability problems in bivariate missingness graphs.  By
bivariate missingness graphs we mean semi-Markovian graphs for two
variables, $X$ and $Y$, and their missingness indicators, $R_X$ and $R_Y$. 
Noting that edges from $\{R_X, R_Y \}$ to $\{X, Y \}$ are not allowed, there
are 9216 such graphs.  We consider only 6144 graphs of which 3072 have the
edge $X \rightarrow Y$ and 3072 do not have an edge between $X$ and $Y$. 
Graphs with the edge  $Y \rightarrow X$ are obtained from the studied graphs
by swapping the roles of $X$ and $Y$.  The maximum number of edges in a
bivariate missingness graph is 12 (when a bidirected edge is counted as a
single edge).

The available theoretical results for missing data problems include a
theorem by \citet{Mohan2013} that gives a sufficient and necessary condition
for the identifiability of the joint distribution $P(\+ V)$ but is
restricted to graphs that do not have edges between the missingness
indicators (Row~9 of Table~\ref{tab:stateofart}).  In our example, 5120
graphs out of 6144 have such edges.  The algorithm by \citet{Shpitser2015}
does not have this restriction but it is not complete as shown by
\citet{bhattacharya2019} (Row~10 of Table~\ref{tab:stateofart}).  It follows
from the results of \citet{bhattacharya2019} that the rules of
Table~\ref{tab:md_rules} are not complete for missing data problems. 
Differently from the theorem by \citet{Mohan2013} and the algorithms by
\citet{Shpitser2015} and \citet{bhattacharya2019}, \dosearch{} can however
address missing data problems where we consider identification of a marginal
or a conditional distribution.  In addition, \dosearch{} can address missing
data problems with multiple input distributions.

The queries $P(X,Y)$, $P(X)$, $P(Y)$, $P(Y \mid X)$ and $P(Y \mid
\doo(X))$ were evaluated using \dosearch{} in these 6144 graphs with the
input distribution $P(X^*, Y^*, R_X, R_Y)$.  The results are summarized by
Venn diagrams in Figure~\ref{fig:m2venn}.  The results are also available as
a data set \texttt{bivariate\_missingness} in the \proglang{R}~package
\pkg{dosearch}.  Using this data set we are able to showcase examples on
non-identifiability and find interesting special cases by direct evaluation
of all possible bivariate missingness graphs.  The first example relates
non-identifiability to the of the number of edges present in the graph.
\begin{example}
Let $K$ denote the number of edges in a bivariate missingness graph that
does not have edge $Y \rightarrow X$.  The joint distribution $P(X,Y)$ is
not identifiable by \dosearch{} if $K>5$, marginal distribution $P(X)$ is
not identifiable by \dosearch{} if \mbox{$K>9$}, marginal distribution
$P(Y)$ and conditional distribution $P(Y \mid X)$ are not identifiable by
\dosearch{} if \mbox{$K>8$}.
\end{example}
The next example specifies the graph with the largest number of edges where
both the joint distribution of $X$ and $Y$ and the causal effect of $X$ on
$Y$ can be identified.
\begin{example} \label{th:jointanddo5edges}
The graph in Figure~\ref{fig:m2examples}(\subref{fig:jointanddo5edges}) is
the only bivariate missingness graph that (i) has edge $X \rightarrow Y$,
(ii) has five edges, and (iii) allows for the identification of $P(X,Y)$ and
$P(Y \mid \doo(X))$ by \dosearch{}.
\end{example}
The third example specifies the graph with the largest number of edges where
the marginal distributions are identifiable while the joint distribution and
the causal effect of $X$ on $Y$ are non-identifiable.
\begin{example} \label{th:marginalsonly5edges}
The graph in Figure~\ref{fig:m2examples}(\subref{fig:marginalsonly5edges})
is the only bivariate missingness graph that (i) has five edges, and (ii)
allows for the identification of $P(X)$ and $P(Y)$, and (iii) does not allow
for the identification of $P(X,Y)$ or $P(Y \mid \doo(X))$ by \dosearch{}. 
No bivariate missingness graph that has more than five edges fulfills the
conditions (ii) and (iii).
\end{example}
\begin{figure}[t!]
  \begin{subfigure}[t]{0.49\textwidth}
    \centering
    \includegraphics[width=1.0\textwidth]{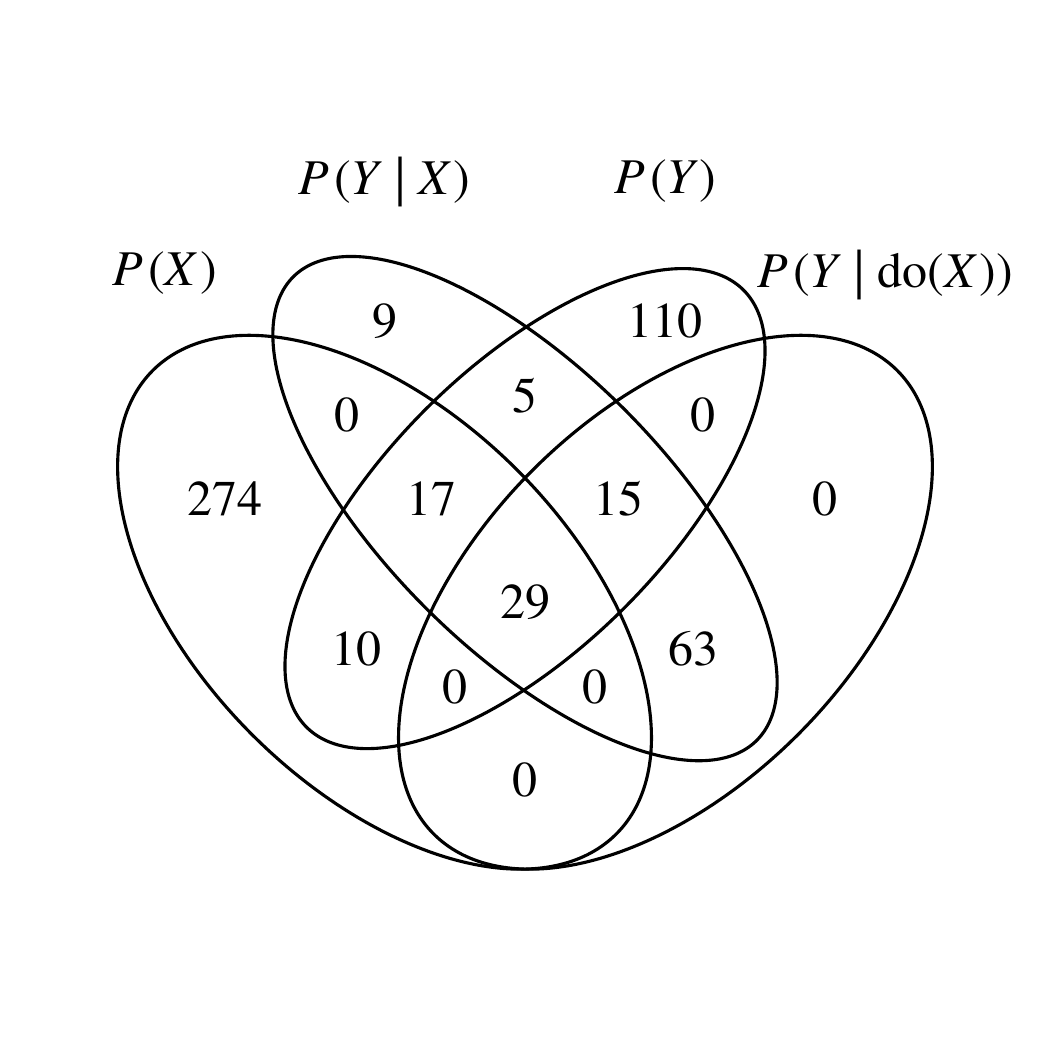}
    \caption{Graphs with arrow $X \rightarrow Y$.}
    \label{fig:m2venn_arrowXtoY}
  \end{subfigure}
  \begin{subfigure}[t]{0.49\textwidth}
    \centering
    \includegraphics[width=1.0\textwidth]{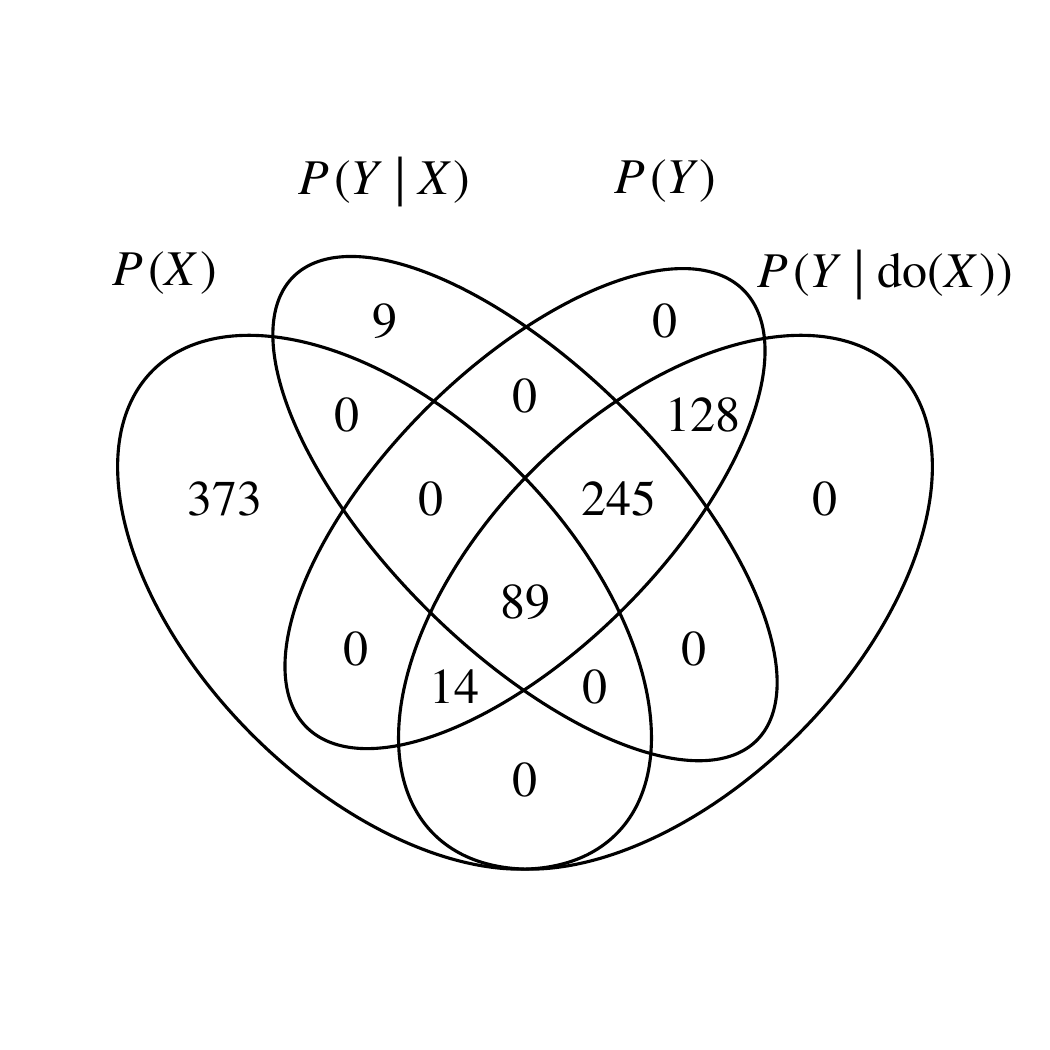}
    \caption{Graphs where there is neither $X \rightarrow Y$ nor $Y \rightarrow X$.}
    \label{fig:m2venn_noarrowXtoY}
  \end{subfigure}
  \caption{Venn diagrams indicating the number of graphs were different
  distributions can be identified \dosearch{}.  The intersection of $P(X)$
  and $P(Y \mid X)$ shows the number of graphs were $P(X,Y)$ can be
  identified.  The total number of possible graphs is 3072 in both cases.}
  \label{fig:m2venn}
\end{figure}
Some interesting examples are shown in Figure~\ref{fig:m2examples}. 
Graphs~(\subref{fig:jointanddo5edges})
and (\subref{fig:marginalsonly5edges}) are the unique graphs that fulfill the
conditions specified in
Examples~\ref{th:jointanddo5edges} and \ref{th:marginalsonly5edges},
respectively.  Graph~(\subref{fig:minimal}) is the graph with the
smallest number of edges
where marginals $P(X)$ and $P(Y)$ can be identified but the joint
distribution $P(X,Y)$ or causal effect $P(Y \mid \doo(X))$ cannot be
identified by \dosearch{}.  In Graph~(\subref{fig:cross1}), $P(X)$, $P(Y)$, $P(X,Y)$ and $P(Y
\mid \doo(X))$ are not identifiable by \dosearch{} but the conditional
distribution $P(Y \mid X)$ can be identified as follows
\begin{equation} \label{eq:cross1and2}
 P(Y \mid X) = \frac{ P(Y \mid R_Y=1) P(X \mid Y, R_X=1, R_Y=1) }
   {\sum_{Y^\prime} P(Y^\prime \mid R_Y=1) P(X \mid Y^\prime, R_X=1,
   R_Y=1)}.
\end{equation}
In Equation~\ref{eq:cross1and2}, the numerator resembles the joint
distribution $P(X,Y \mid R_X = 1,R_Y = 1)$ but is different because $Y$
and $R_X$ are not independent.  The denominator is the marginal of this
pseudo joint distribution.
In Graph~(\subref{fig:cross2}), $P(X)$, $P(Y)$ and $P(X,Y)$ are not identifiable by
\dosearch{} but $P(Y \mid X)$ and $P(Y \mid \doo(X))$ are identifiable
and can be both estimated with Equation~\ref{eq:cross1and2}.  In Graph~(\subref{fig:openbox}),
$P(X,Y)$, $P(X)$ and $P(Y \mid \doo(X))$ are not identifiable by
\dosearch{} but $P(Y)$ and $P(Y \mid X)$ can be identified as follows
\begin{align}
 P(Y) &= \sum_{R_X,X^*} P(Y\mid X^*, R_X, R_Y=1) P(R_X,X^*),
  \label{eq:openboxY}\\
 P(Y \mid X) &= P(Y \mid X, R_X=1, R_Y=1) \nonumber
\end{align}
In Equation~\ref{eq:openboxY}, the summation also goes over the cases
where $X^* = \textrm{NA}$ and the distribution of $Y$ must be estimated also
on the condition that $X$ is not observed.
\begin{figure}[t!]
  \centering
  \begin{subfigure}[t]{0.3\textwidth}
    \centering
    \begin{tikzpicture}[scale=1.0]
    \node [dot = {0}{0}{X}{below left}] at (0,0) {};
    \node [dot = {0}{0}{Y}{below right}] at (2,0) {};
    \node [dot = {0}{0}{R_X}{above left}] at (0,2) {};
    \node [dot = {0}{0}{R_Y}{above right}] at (2,2) {};
    \draw [->] (R_X) -- (R_Y) ;
    \draw [->] (X) -- (R_Y) ;
    \draw [->] (X) -- (Y) ;
    \draw [<->,dashed] (X) to [bend right=30] (R_Y) ;
    \draw [<->,dashed] (R_X) to [bend left=30] (R_Y) ;
    \end{tikzpicture}
    \caption{}
    \label{fig:jointanddo5edges}
  \end{subfigure}
  \begin{subfigure}[t]{0.3\textwidth}
    \centering
    \begin{tikzpicture}[scale=1.0]
    \node [dot = {0}{0}{X}{below left}] at (0,0) {};
    \node [dot = {0}{0}{Y}{below right}] at (2,0) {};
    \node [dot = {0}{0}{R_X}{above left}] at (0,2) {};
    \node [dot = {0}{0}{R_Y}{above right}] at (2,2) {};
    \draw [->] (R_Y) -- (R_X) ;
    \draw [->] (X) -- (Y) ;
    \draw [<->,dashed] (X) to [bend right=30] (Y) ;
    \draw [<->,dashed] (Y) to [bend left=30] (R_X) ;
    \draw [<->,dashed] (R_X) to [bend left=30] (R_Y) ;
    \end{tikzpicture}
    \caption{}
    \label{fig:marginalsonly5edges}
  \end{subfigure}
  \begin{subfigure}[t]{0.3\textwidth}
    \centering
    \begin{tikzpicture}[scale=1.0]
    \node [dot = {0}{0}{X}{below left}] at (0,0) {};
    \node [dot = {0}{0}{Y}{below right}] at (2,0) {};
    \node [dot = {0}{0}{R_X}{above left}] at (0,2) {};
    \node [dot = {0}{0}{R_Y}{above right}] at (2,2) {};
    \draw [->] (X) -- (Y) ;
    \draw [<->,dashed] (Y) to [bend left=30] (R_X) ;
    \end{tikzpicture}
    \caption{}
    \label{fig:minimal}
  \end{subfigure}

  \begin{subfigure}[t]{0.3\textwidth}
    \centering
    \begin{tikzpicture}[scale=1.0]
    \node [dot = {0}{0}{X}{below left}] at (0,0) {};
    \node [dot = {0}{0}{Y}{below right}] at (2,0) {};
    \node [dot = {0}{0}{R_X}{above left}] at (0,2) {};
    \node [dot = {0}{0}{R_Y}{above right}] at (2,2) {};
    \draw [->] (R_Y) -- (R_X) ;
    \draw [->] (X) -- (R_Y) ;
    \draw [->] (Y) -- (R_X) ;
    \draw [->] (X) -- (Y) ;
    \draw [<->,dashed] (X) to [bend right=30] (Y) ;
    \end{tikzpicture}
    \caption{}
    \label{fig:cross1}
  \end{subfigure}
  \begin{subfigure}[t]{0.3\textwidth}
    \centering
    \begin{tikzpicture}[scale=1.0]
    \node [dot = {0}{0}{X}{below left}] at (0,0) {};
    \node [dot = {0}{0}{Y}{below right}] at (2,0) {};
    \node [dot = {0}{0}{R_X}{above left}] at (0,2) {};
    \node [dot = {0}{0}{R_Y}{above right}] at (2,2) {};
    \draw [->] (R_Y) -- (R_X) ;
    \draw [->] (X) -- (R_Y) ;
    \draw [->] (Y) -- (R_X) ;
    \draw [->] (X) -- (Y) ;
    \draw [<->,dashed] (X) to [bend right=30] (R_Y) ;
    \end{tikzpicture}
    \caption{}
    \label{fig:cross2}
  \end{subfigure}
  \begin{subfigure}[t]{0.3\textwidth}
    \centering
    \begin{tikzpicture}[scale=1.0]
    \node [dot = {0}{0}{X}{below left}] at (0,0) {};
    \node [dot = {0}{0}{Y}{below right}] at (2,0) {};
    \node [dot = {0}{0}{R_X}{above left}] at (0,2) {};
    \node [dot = {0}{0}{R_Y}{above right}] at (2,2) {};
    \draw [->] (R_X) -- (R_Y) ;
    \draw [->] (X) -- (R_X) ;
    \draw [->] (X) -- (Y) ;
    \draw [<->,dashed] (X) to [bend right=30] (Y) ;
    \end{tikzpicture}
    \caption{}
    \label{fig:openbox}
  \end{subfigure}
  \caption{Missingness graphs used as example cases.  Proxy variables are
    omitted for clarity.}
   \label{fig:m2examples}
\end{figure}
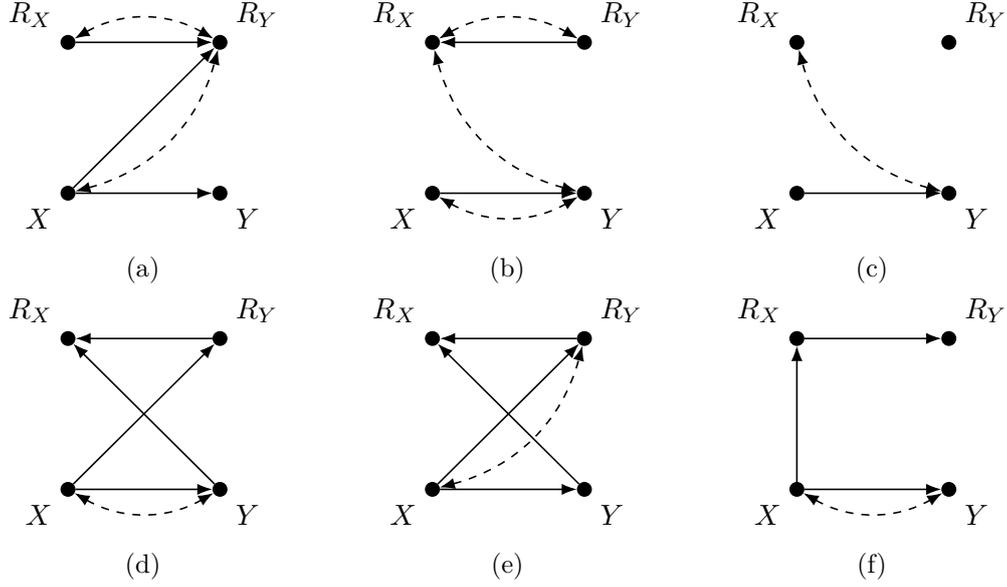

\subsection{Causal inference under case-control design}

Case-control design \citep{Breslow1996casecontrol} is commonly used in
epidemiology to study risk factors of rare diseases.  In the basic setup, a
fixed number of disease cases and a fixed number of controls are selected
for the risk factor measurements.  When the disease is rare, this design
leads to substantial savings in the sample size compared to simple random
sampling.  Figure~\ref{fig:ccexample}(\subref{fig:simplecc}) shows the
missingness graph for a situation where the inclusion to the study
(indicator $R_Y$) depends on the disease endpoint $Y$.  The risk factors $X$
are measured for the subset $R_Y=1$ but occasionally the values are missing
(indicator $R_X$).  It is immediately seen that neither the causal effect
$P(Y \mid \doo(X))$ nor conditional distribution $P(Y \mid X)$ can be
identified because of the edge $Y \rightarrow R_Y$.  However, if the
prevalence of the disease in the population, i.e.,~the marginal distribution
$P(Y)$, is known, the causal effect $P(Y \mid \doo(X))$ can be identified. 
The result is provided by \dosearch{}
\begin{equation} \label{eq:simplecc}
 P(Y \mid \doo(X)) = \frac{P(Y)P(X \mid Y, R_Y=1,
    R_X=1)}{\sum_{Y^\prime} P(Y^\prime)P(X \mid Y^\prime, R_Y=1, R_X = 1)}.
\end{equation}
In typical applications response $Y$ is binary but in the non-parametric
formula of Equation~\ref{eq:simplecc} response can be discrete or
continuous.
\begin{figure}[t!]
\centering
\begin{subfigure}[t]{0.40\textwidth}
 \centering
  \begin{tikzpicture}[scale=1.1]
\node [dot = {0}{0}{X}{below}] at (0,0) {};
\node [dot = {0}{0}{Y}{below}] at (1.6,0) {};
\node [dot = {0}{0}{R_Y}{left}] at (1.8,1) {};
\node [dot = {0}{0}{R_X}{left}] at (2.0,2) {};
\draw [->] (X) -- (Y) ;
\draw [->] (Y) -- (R_Y) ;
\draw [->] (R_Y) -- (R_X) ;
\end{tikzpicture}
  \caption{Basic case-control design.}
  \label{fig:simplecc}
\end{subfigure}
\hfill
\begin{subfigure}[t]{0.55\textwidth}
\centering
 \begin{tikzpicture}[scale=1.1]
\node [dot = {0}{0}{X}{below}] at (0,0) {};
\node [dot = {0}{0}{Z}{below}] at (1.6,0) {};
\node [dot = {0}{0}{Y}{below}] at (3.2,0) {};
\node [dot = {-0.1}{0}{R_Y}{left}] at (3.4,1) {};
\node [dot = {-0.1}{0}{R_X}{left}] at (3.6,2) {};
\node [dot = {0}{0}{R_Z}{right}] at (4.6,2) {};
\draw [->] (X) -- (Z) ;
\draw [->] (Z) -- (Y) ;
\draw [->] (Y) -- (R_Y) ;
\draw [->] (R_Y) -- (R_X) ;
\draw [->] (R_Y) -- (R_Z) ;
\draw [<->,dashed] (X) to [bend left=30] (Y) ;
\draw [<->,dashed] (R_X) to [bend right=30] (R_Y) ;
\draw [<->,dashed] (R_X) to [bend left=30] (R_Z) ;
\draw [<->,dashed] (R_Z) to [bend left=30] (R_Y) ;
\end{tikzpicture}
 \caption{Case-control design for the front-door situation.}
  \label{fig:frontdoorcc}
\end{subfigure}
 \caption{Missingness graph for the case-control examples.}
  \label{fig:ccexample}
\end{figure}
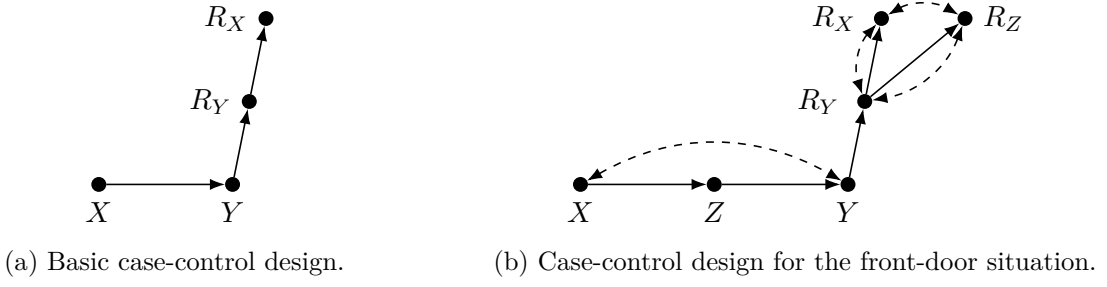
A more complicated example is shown in
Figure~\ref{fig:ccexample}(\subref{fig:frontdoorcc}) where the causal effect
of risk factor $X$ on disease endpoint $Y$ fulfills the front-door criterion
\citep{pearl1995causal} with respect to mediator $Z$ and the data are
collected from a case-control design where the selection depends $Y$ and
there is occasional item non-response in $X$ and $Z$.  We observe data
$P(Y^*,X^*,Z^*,R_Y,R_X,R_Z)$ and know the marginal distribution $P(Y)$ from
other sources.  Applying \dosearch{} we obtain the result
\begin{equation}
\begin{aligned}
&  P(Y \mid \doo(X)) = \\
&\quad \sum_Z \left[ \frac{\sum_{Y^\prime} P(Y^\prime)P(X,Z \mid
Y^\prime,R_X = 1, R_Y = 1, R_Z = 1)}{\sum_{Z^\prime,Y^\prime}
P(Y^\prime)P(X,Z^\prime \mid Y^\prime,R_X = 1, R_Y = 1, R_Z = 1)}\,
\right.  \times \\
&\qquad \sum_{X^\prime} \left( \sum_{Y^\prime,Z^\prime}
P(Y^\prime)P(X^\prime,Z^\prime \mid Y^\prime,R_X = 1, R_Y = 1, R_Z = 1)\,
\right.  \times \\
&\quad\qquad \left.  \left.  \frac{P(Y)P(X^\prime,Z \mid Y, R_X = 1, R_Y =
1,R_Z = 1)}{\sum_{Y^\prime} P(Y^\prime)P(X^\prime,Z \mid Y^\prime,R_X =
1,R_Y = 1,R_Z = 1)} \vphantom{\sum_{Z^\prime}}\right) \right].
\end{aligned}
\label{eq:frontdoorcc}
\end{equation}
Expression~\ref{eq:frontdoorcc} follows the general structure of the
front-door adjustment
\[
 P(Y \mid \doo(X)) = \sum_Z P(Z \mid X) \sum_{X^\prime} P(X^\prime) P(Y
   \mid X^\prime,Z),
\]
where
\begin{align*}
  P(Z \mid X) &= \frac{\sum_{Y^\prime} P(Y^\prime)P(X,Z \mid
     Y^\prime,R_X = 1, R_Y = 1, R_Z = 1)}{\sum_{Z^\prime,Y^\prime}
     P(Y^\prime)P(X,Z^\prime \mid Y^\prime,R_X = 1, R_Y = 1, R_Z = 1)}, \\
  P(X) &= \sum_{Y^\prime,Z^\prime} P(Y^\prime)P(X,Z^\prime \mid
     Y^\prime,R_X = 1, R_Y = 1, R_Z = 1), \\
  P(Y \mid X,Z) &= \frac{P(Y)P(X,Z \mid Y,R_X = 1,R_Y = 1,R_Z =
     1)}{\sum_{Y^\prime} P(Y^\prime)P(X,Z \mid Y^\prime,R_X = 1,R_Y = 1,R_Z =
     1)}.
\end{align*}
Note that $P(X,Y,Z) = P(Y)P(X,Z \mid Y,R_X = 1,R_Y = 1,R_Z = 1)$.  In
\citep{Karvanen2015studydesign}, a similar example was studied assuming that
$X$, $Z$ and $Y$ are binary but in Expression~\ref{eq:frontdoorcc} there
are no such restrictions.  This factorization can be obtained in
\proglang{R} as follows
\begin{CodeChunk}
\begin{CodeInput}
R> data <- "
+    p(x*,y*,z*,r_x,r_y,r_z)
+    p(y)
+  "
R> graph <- "
+    x -> z
+    z -> y
+    y -> r_y
+    x <-> y
+    r_y -> r_x
+    r_y -> r_z
+    r_y <-> r_x
+    r_y <-> r_z
+    r_z <-> r_x
+  "
R> md <- "r_x : x, r_y : y, r_z : z"
R> query1 <- "p(z|x)"
R> query2 <- "p(x)"
R> query3 <- "p(y|x,z)"
R> dosearch(data, query1, graph, missing_data = md)
\end{CodeInput}
\begin{CodeOutput}
\frac{\sum_{y}\left(p(y)p(x,z|r_x = 1,y,r_y = 1,r_z = 1)\right)}
  {\sum_{z} \sum_{y}\left(p(y)p(x,z|r_x = 1,y,r_y = 1,r_z = 1)\right)}
\end{CodeOutput}
\begin{CodeInput}
R> dosearch(data, query2, graph, missing_data = md)
\end{CodeInput}
\begin{CodeOutput}
\sum_{y,z}\left(p(y)p(x,z|r_x = 1,y,r_y = 1,r_z = 1)\right)
\end{CodeOutput}
\begin{CodeInput}
R> dosearch(data, query3, graph, missing_data = md)
\end{CodeInput}
\begin{CodeOutput}
\frac{\left(p(y)p(x,z|r_x = 1,y,r_y = 1,r_z = 1)\right)}
  {\sum_{y} \left(p(y)p(x,z|r_x = 1,y,r_y = 1,r_z = 1)\right)} 
\end{CodeOutput}
\end{CodeChunk}

\section{Discussion} \label{sect:discussion}

The presented algorithm, \dosearch{}, removes the need for manual
application of do-calculus, which is time-consuming and prone to errors. 
Systematic analyses such as the one in Section~\ref{sect:systematic} are
practically unreachable with manual application of do-calculus.  Superiority
of \dosearch{} over a simple forwards breadth-first search was attained
through a combination of a search heuristic and a reduction of the search
space.  Some further approaches were attempted but later discarded as
non-beneficial.  These include caching separation criteria that hold in the
graph after they are first evaluated, pre-computing valid subsets for each
subset size and enumerating subsets in an order of increasing cardinality.

As the simulations showed, our intuitive heuristic yielded significant
improvements in search performance.  The proximity function defined in
Section~\ref{sect:improvements} uses only the information contained in the
distributions themselves.  One approach could be to also take the structure
of the graph into account in the proximity function.  Further study is
needed for finding a heuristic that performs well when missing data
mechanisms are present in the graph.

The scalability of \dosearch{} is limited due to vast search space of
possibly identified causal effects.  Currently, algorithms with polynomial
complexity currently exist only for the simpler problems (see
Table~\ref{tab:stateofart}).  However, based on the simulation results,
\dosearch{} solves identifiability problems in graphs of ten~vertices 
in under two minutes on average.  By our observation, graphs typically analyzed
in literature related to identifiability problems have fewer vertices.  
The theoretical computational complexity of the general form of the causal
identifiability problem defined in Section~\ref{sect:problem} remains an
important and interesting question.

The search could also be used to obtain formulas that are in some sense
simpler than those produced by existing identifiability algorithms.  A
simplification algorithm by \citet{Tikka:simplifying} functions as a
post-processing step after the identifying formula has already been obtained
by the ID algorithm.  Given a measure of simplicity, the search heuristic
could be adjusted to find simple formulas directly without resorting to
separate simplification procedures.  In some specific scenarios, such as the
standard causal effect identifiability problem, an approach known as pruning
\citep{Tikka:pruning} could be incorporated into the search.  Pruning refers
to the removal of vertices from the graph, that are not required for
determining identifiability.

Finally we note that identifiability has also been studied under the
assumption that the functional relationships depicted by the causal model
are linear \citep{angrist1996,zander2016,chen2017} or non-parametric with
additive error terms \citep{peters2014,pena2017} and when the causal graph
is not completely known
\citep{IDA,ENTNER,hyttinen2015,perkovic2015,MALINSKY,jaber2018}.  Extending
the search in these directions is an interesting line of future research.

\section{Conclusion} \label{sect:conclusion}

We presented \dosearch{}: a do-calculus based search capable of solving
identifiability problems for which no known solutions exist.  This
contribution is especially useful for researchers working in the field of
causal inference to confirm theoretical results or to find counterexamples
to identifiability claims.  In practical terms, the search can also provide
solutions to complicated problems such as combining transportability and
selection bias, recovering from multiple bias sources or identifying causal
quantities in the presence of missing data that cannot be solved by any
other existing method.  The \proglang{R}~package \pkg{dosearch} providing an
implementation of \dosearch{} is available on CRAN.

\section*{Acknowledgments}

This work belongs to the thematic research area ``Decision analytics
utilizing causal models and multiobjective optimization'' (DEMO) supported
by Academy of Finland (grant number 311877).  AH was supported by Academy of
Finland through grant 295673.  The authors wish to acknowledge CSC -- IT
Center for Science, Finland, for computational resources.

\bibliography{ref}

\end{document}